\documentclass[11pt]{article}
\usepackage[utf8]{inputenc}
 \usepackage{graphicx} 
\usepackage{geometry}
\usepackage{authblk}  

\usepackage{color}
\usepackage{booktabs}
\usepackage{pifont}
 
\usepackage[T1]{fontenc}
\usepackage{lmodern}
\usepackage{amsmath,amssymb,amsthm}

\usepackage{hyperref} 
\usepackage{cleveref}

\newcommand{\R}{\mathbb{R}}
\newcommand{\Z}{\mathbf{Z}}
\newcommand{\B}{\mathbf{B}}

\renewcommand*{\d}{\mathop{}\!\mathrm{d}}
\renewcommand{\eqref}[1]{Eq.~(\ref{#1})}
\newcommand{\T}{T}

\usepackage{listings}
\newcommand{\norm}[1]{\left\lVert#1\right\rVert}
  
\newtheorem{theorem}{Theorem}
\newtheorem{assumption}{Assumption}
\newtheorem{lemma}{Lemma}
\newtheorem{proposition}{Proposition}

\newcommand{\ourmethodname}{UIDD} 

\usepackage{ulem}
\title{\LARGE\bfseries Identifiable learning of dissipative dynamics}
\author[a]{Aiqing Zhu}
\author[b]{Beatrice W. Soh}
\author[c,1]{Grigorios A. Pavliotis}
\author[a,d,1]{Qianxiao Li}

\affil[a]{Department of Mathematics, National University of Singapore, 10 Lower Kent Ridge Road, 119076, Singapore}
\affil[b]{Department of Chemical and Biomolecular Engineering, National University of Singapore, 4 Engineering Drive 4, 117585, Singapore}
\affil[c]{Department of Mathematics, Imperial College London, London, SW7 2AZ, United Kingdom}
\affil[d]{Institute for Functional Intelligent Materials, National University of Singapore, 4 Science Drive 2, 117544, Singapore}

\affil[1]{\textit{Corresponding authors:} \href{mailto:g.pavliotis@imperial.ac.uk}{g.pavliotis@imperial.ac.uk}, \href{mailto:qianxiao@nus.edu.sg}{qianxiao@nus.edu.sg}}
\date{\today}
\begin{document} 
\maketitle 
\begin{abstract}
{
Complex dissipative systems appear across science and engineering, from polymers and active matter to learning algorithms. These systems operate far from equilibrium, where energy dissipation and time irreversibility govern their behavior but are difficult to quantify from data. 
Here, we introduce a universal and identifiable neural framework that learns dissipative stochastic dynamics directly from trajectories while ensuring interpretability, expressiveness, and uniqueness. 
Our method identifies a unique energy landscape, separates reversible from irreversible motion, and allows direct computation of the entropy production, providing a principled measure of irreversibility and deviations from equilibrium.
Applications to polymer stretching in elongational flow and to stochastic gradient Langevin dynamics reveal new insights, including super-linear scaling of barrier heights and sub-linear scaling of entropy production rates with the strain rate, and the suppression of irreversibility with increasing batch size. Our methodology thus establishes a general, data-driven framework for discovering and interpreting non-equilibrium dynamics.
}
\end{abstract}
\section{Introduction}
Complex dissipative dynamical systems are ubiquitous in both natural phenomena and engineering applications~\cite{chaudhari2018stochastic,cross1993pattern,marchetti2013hydrodynamics, nartallo2026nonequilibrium,ottinger2005beyond, perkins1997single, welling2011bayesian}, ranging from active matter and polymer physics to stochastic optimization algorithms in deep learning. While many foundational phenomena near or at equilibrium admit clear, mechanistically transparent descriptions, non-equilibrium systems lack a universally accepted canonical form and are instead described through empirical laws. Their defining features---energy dissipation, probability currents, and time irreversibility---are central to system function and performance~\cite{battle2016broken,boffi2024deep, jarzynski2012equalities, seifert2012stochastic}, yet they are difficult to quantify directly from observations and even harder to encode reliably in mathematical models. Building dynamical models that both predict accurately across conditions and reveal the organizing principles of non-equilibrium dynamics therefore remains a central challenge.

Recent advances in artificial intelligence and machine learning have enabled powerful data-driven modeling of dynamical processes~\cite{brunton2022data,brunton2016discovering,chen2025due,chen2018neural, gaskin2024neural,gaskin2023neural, gaskin2024inferring, karniadakis2021physics,liu2021machine}. A particularly promising direction involves guiding neural networks with physical priors that constrain the hypothesis space, leading to predictive models that are both accurate and interpretable.
Different from methods based on (sparse) identification from a dictionary~\cite{brunton2016discovering,champion2019data}, this approach defines the hypothesis space using a continuous functional template that prescribes the intrinsic structure of admissible dynamics. Notable efforts in this direction include models based on Hamiltonian dynamics for conservative systems~\cite{bertalan2019learning,greydanus2019hamiltonian, zhu2023machine}, and those built upon the Onsager principle~\cite{chen2024constructing,yu2021onsagernet} or the GENERIC framework~\cite{lee2021machine, zhang2022gfinns} for dissipative, non-equilibrium dynamics.

However, when such approaches are implemented, a fundamental dilemma arises. On the one hand, the hypothesis space must be expressive enough to capture a wide range of dissipative dynamics. On the other hand, it must be sufficiently constrained to guaranty model identifiability---ensuring, for example, that the learned energy function is unique up to an additive constant. Without identifiability, the learned representations cannot be used
to reliably infer system behavior. 
This is particularly critical when learning from a family of datasets, where uniquely defined quantities are essential for systematic comparison across parameter space.
Current methodologies for learning non-equilibrium dissipative dynamics do not satisfy these competing criteria simultaneously, as detailed in Table \ref{tab:non-uniqueness}.

\begin{table}[htbp]
    \centering
    \resizebox{\linewidth}{!}{
    \begin{tabular}{l|l| c | c l }
        \toprule
        Methods & Dynamics & Universal & Identifiable & Non-trivial invariant transformation \\
        \midrule
        Onsager principle~\cite{doi2011onsager,onsager1931reciprocal} & 
        $\dot{\Z}_t = -M(\Z_t)\nabla V(\Z_t)$ & \ding{55} & \ding{55} &
        $(V, M) \to (\varphi(V), M/\varphi')$ \\
        
        OnsagerNets~\cite{chen2024constructing,yu2021onsagernet} & 
        $\dot{\Z}_t = -[M(\Z_t)+W(\Z_t)]\nabla V(\Z_t)$ & \ding{51}& \ding{55} &
        $(V, M, W) \to (\varphi(V), M/\varphi', W/\varphi')$ \\
        
        GENERIC ~\cite{lee2021machine,zhang2022gfinns} & 
        $\dot{\Z}_t = L(\Z_t) \nabla E(\Z_t) + M(\Z_t) \nabla S(\Z_t)$ & \ding{51}&\ding{55} &
        $(E,S,L,M) \to (\varphi(E), \varphi(S), L/\varphi', M/\varphi')$ \\
        
        Stat-PINNs~\cite{huang_al_2025} & $\dot{\Z}_t = -M(\Z_t)\nabla V(\Z_t) + \sqrt{2\varepsilon M(\Z_t) } \dot{\B}_t$ & \ding{55}&  \ding{51} &None \\

        Ours & 
        \eqref{eq: OnsagerHHD} & \ding{51}&  \ding{51} & None\\
        \bottomrule
    \end{tabular}
    }
    \caption{Illustration of common frameworks for learning dissipative dynamics. The dynamics is invariant under the corresponding transformation, where $\varphi$ is any strictly increasing differentiable function. This invariance indicates that applying the method to the same dataset may yield multiple valid but distinct potential functions, thereby leading to non-identifiability.
    We also mention the recent work~\cite{huang_al_2025} in which Stat-PINNs were introduced to discover thermodynamic structures of equilibrium dissipative systems from short-time particle simulations. This work also addresses the uniqueness issue, but the model is formulated to represent only equilibrium systems.}
    \label{tab:non-uniqueness}
\end{table}

We introduce Universal and Identifiable Dissipative Dynamics
({\ourmethodname}), a neural framework for learning dissipative, non-equilibrium stochastic dynamics directly from trajectory data. The hallmark of {\ourmethodname} is the physical interpretability of its identified components: we construct a potential function which admits a direct interpretation as the negative logarithm of the stationary density, and a drift function that naturally decomposes into time-reversible and time-irreversible parts. These properties enable direct access to irreversibility and entropy production rate (EPR)~\cite{da2023entropy,jiang2004mathematical}, allowing the model to diagnose departures from equilibrium and to quantify the underlying drivers of non-equilibrium behavior directly from data. Most importantly, our method achieves mathematical identifiability without sacrificing generality. We prove that {\ourmethodname} can represent any stochastic differential equation with a unique invariant distribution---covering the full class of stable and ergodic dissipative systems---and that its components are uniquely determined.

We demonstrate the capabilities of {\ourmethodname} by learning the dynamics of a 300-bead polymer chain stretching in elongational flow, a system with up to 900 degrees of freedom. Using coarse-grained representations learned from microscopic data, we show that {\ourmethodname} reliably reconstructs consistent potential energy landscapes and reveals that the barrier height increases super-linearly with the strain rate. In addition, {\ourmethodname} further quantifies both local and global EPR, uncovering a sub-linear scaling of EPR with strain rate and confirming that the system operates far from equilibrium. These findings were previously inaccessible due to non-identifiability. We further validate the non-equilibrium nature with single-molecule experimental data.

To further demonstrate its broad applicability, we also employ {\ourmethodname}
to investigate stochastic gradient Langevin dynamics, a widely used
algorithm in Bayesian inference and non-convex optimization~\cite{welling2011bayesian}. We quantify how
mini-batching introduces state-dependent noise that breaks detailed balance and
drives the system out of equilibrium. Our analysis reveals that the EPR
decreases with the batch size before plateauing at larger batches, indicating a
suppression of non-equilibrium behavior as stochasticity is reduced.
In addition, we find that the sampling error follows a consistent trend, suggesting that the EPR can serve as a physically-grounded diagnostic for the sampling quality.
Together, these applications illustrate the
versatility of {\ourmethodname} in uncovering and quantifying non-equilibrium
behavior in complex stochastic processes across physically and algorithmically
generated non-equilibrium systems.

\section{Results}

\subsection{{\ourmethodname} Architecture}
Universal and Identifiable Dissipative Dynamics ({\ourmethodname}) is a neural framework for learning dissipative dynamics from discrete trajectory data. Its key innovation is a universal and identifiable dynamical formulation whose evolution is defined by the following It\^o stochastic differential equation (SDE):
\begin{equation}\label{eq: OnsagerHHD}
\begin{aligned}
\dot{\Z}_t=& -[M(\Z_t) + W(\Z_t)]\nabla V(\Z_t) + \nabla\cdot M(\Z_t)  + \nabla\cdot W(\Z_t) + \sigma(\Z_t) \dot{\B}_t,
\end{aligned}
\end{equation}
where $V(\cdot)$ is a scalar potential, $M(\cdot)$ and $W(\cdot)$ are $D \times D$ matrix-valued functions. $M(z) = \sigma(z)\sigma^{\top}(z)/2$ is symmetric positive definite, while $W(z)$ is a banded antisymmetric matrix whose non-zero entries are confined to the first sub-diagonal and its corresponding counterparts.

In this model, the diffusion matrix $\sigma(\cdot)$ together with the white
noise process $\dot{\B}_t$ models the thermal fluctuations. The symmetric matrix
$M(\cdot)$ characterizes the energy dissipation of the system and is constructed
in such a way that the model satisfies the fluctuation-dissipation relation,
where we set $M = \sigma \sigma^\top / 2$.
The antisymmetric matrix $W(\cdot)$ models the
time-irreversible phenomena, enabling the model to account for non-equilibrium
effects. Notably, the proposed sub-diagonal form of $W$ contains $D-1$
independent scalar functions, thereby guaranteeing that the number of scalar functions to be modeled in the drift coincides with the system dimension.
As shown later, this choice achieves simultaneously full expressive capacity and identifiability,
while also reducing computational overhead.

We employ a data-driven, machine-learning approach to reconstruct the stochastic dynamics
in~\eqref{eq: OnsagerHHD}. The unknown functions $W$, $V$, and $\sigma$ are each parameterized as trainable deep neural networks. Essential physical priors are encoded into the neural network, including the positive definiteness of $M=\sigma \sigma^{\top}/2$ and the integrability of $\rho=e^{-V}$,
which ensures the existence of the invariant distribution.
The model is trained on a dataset of multiple independent trajectories by maximizing the likelihood of the observed data.
Once trained, identifiability allows for subsequent analysis of the dynamical properties of the system.
A schematic of this computational framework is provided in Fig.~\ref{fig:overview}, with details of the network architecture and learning algorithm provided in Methods~\ref{sec:methods}.
\begin{figure*}[!thb]
    \centering
    \includegraphics[width=0.95\linewidth]{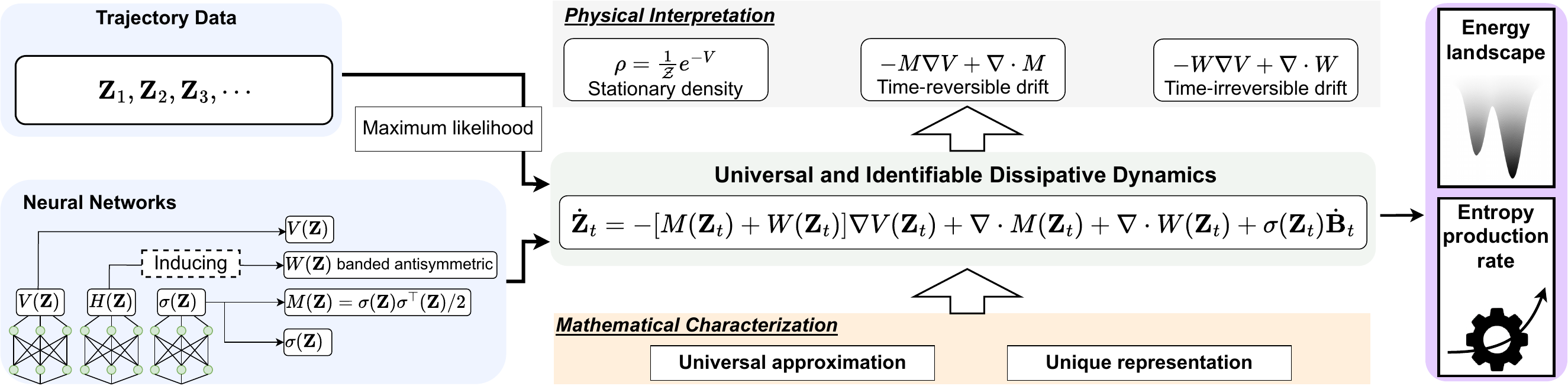}
    \caption{Overview of the {\ourmethodname}.
(Blue) Discrete trajectory data serves as training input, and the model is built upon neural networks.
(Green) Through a designed parameterization and likelihood-based training on this trajectory data, the framework constructs the {\ourmethodname} dynamical model.
(Orange) Mathematical analysis establishes that {\ourmethodname} possesses both universal approximation capabilities and unique representation properties.
(Grey) The components of the model can be interpreted as thermodynamically meaningful elements, including the stationary density, the time-reversible drift, and the time-irreversible drift.
(Purple) Leveraging this physical interpretability and the identifiability, {\ourmethodname} enables the computation of the energy landscape and the EPR.
  }
    \label{fig:overview}
\end{figure*}

A comparison with existing dissipative dynamics models \cite{chen2024constructing, huang_al_2025, lee2021machine,yu2021onsagernet, zhang2022gfinns} is provided in Methods~\ref{app:Existing Models}. Rather than ad-hoc modifications, {\ourmethodname} emerges from the combination of the generalized Onsager principle~\cite{yu2021onsagernet},
the generalized Helmholtz decomposition~\cite{da2023entropy},
and the essentially Hamiltonian decomposition~\cite{feng1995volume}, see Methods~\ref{app:Mathematical Characterization}.
In {\ourmethodname}, the potential $V$ specifies the stationary density; the drift admits an equivalent decomposition into time-reversible and time-irreversible components, enabling computation of the entropy production rate (EPR); and the resulting framework yields a universal and identifiable extension of existing dissipative models.

\paragraph{Stationary Density}
To establish the connection between the stationary density, the drift and diffusion components and the Helmholtz decomposition, we first write
\begin{equation*}
 f = -[M + W ]\nabla V + \nabla\cdot M  + \nabla\cdot W , \quad M  = \sigma \sigma^{\top}/2,
\end{equation*}
and view~\eqref{eq: OnsagerHHD} as
\begin{equation*}
\dot{\Z}_t= f(\Z_t) + \sigma(\Z_t) \dot{\B}_t.
\end{equation*}
The probability density $\rho_t$ of the solution $\Z_t$ satisfies the Fokker-Planck equation (FPE)~\cite{pavliotis2014stochastic}:
\begin{equation*}
\partial_t \rho_t+ \nabla \cdot (f\rho_t - \nabla \cdot (M \rho_t))=0.
\end{equation*}
When the system has reached a statistical steady state, so that $\rho_t=\rho$,
then the above FPE reduces to its stationary form~\cite{pavliotis2014stochastic}:
\begin{equation*}
\nabla \cdot (f\rho - \nabla \cdot (M \rho))=0.
\end{equation*}
Stationary densities of general SDEs are typically intractable.
However, for the {\ourmethodname}, the stationary density admits a closed-form expression (Methods~\ref{app: stationary density}):
\begin{equation*}\label{eq: stationary density}
\rho(\Z) = \frac{1}{\mathcal{Z}} e^{-V(\Z)},
\qquad
\mathcal{Z} = \int e^{-V(\Z)} \d\Z.
\end{equation*}

\paragraph{Time-reversal and Entropy Production}
For a general time-homogeneous Markov process $\Z_t$ on $t\in[0,T]$, its time-reversed process $\bar{\Z}_t$ is defined as $\bar{\Z}_t = \Z_{T-t}$. For {\ourmethodname}, the structured dynamics enables the explicit identification of the time-reversed dynamics. Specifically, when the forward process $\Z_t$ is stationary with respect to $\rho$, its time-reversed process $\bar{\Z}_t$ is the solution to the SDE of the form (Methods~\ref{app:Time-reversal and Entropy Production})~\cite{da2023entropy, jiang2004mathematical}:
\begin{equation*}
  \dot{\bar{\Z}}_t= \bar{f}(\bar{\Z}_t) + \bar{\sigma}(\bar{\Z}_t) \dot{\B}_t,
\end{equation*}
where the diffusion reads $\bar{\sigma}=\sigma$ and the drift reads
\begin{equation*}
\bar{f} = \left\{\begin{aligned}
    &-[M - W]\nabla V + \nabla\cdot M -  \nabla\cdot W,\ &&\text{when } \rho(\Z)>0,\\
    & -f\ &&\text{when } \rho(\Z)=0.
\end{aligned} \right.
\end{equation*}
In other words, the governing function for the time-reversed process is obtained by leaving the $M$-related part unchanged while flipping the sign of the $W$-related part.
It follows that the full dynamics of {\ourmethodname} can be split into time-reversible and time-irreversible components, which correspond exactly to the contributions from the matrices $M$ and $W$, respectively.
\begin{align}
&\text{Time-reversible: }  && f_{\text{rev}} = - M\nabla V+ \nabla\cdot M,\label{eq:rev}\\
&\text{Time-irreversible: } && f_{\text{irr}} = - W\nabla V + \nabla\cdot W.\label{eq:irr}
\end{align}
Leveraging the physical interpretations, we are now able to calculate the entropy production
using {\ourmethodname}.

The global EPR, $\dot{S}_{\text{tot}}$, quantifies the irreversibility of a stochastic process as the degree to which its trajectories are statistically distinguishable from their time-reversed counterparts.
For a time-homogeneous Markov process, it is defined as the Kullback-Leibler (KL) divergence per unit time between the forward path measure $\mathbf{P}$ and the time-reversed path measure $\bar{\mathbf{P}}$~\cite{crooks1999entropy, jarzynski2012equalities, jiang2004mathematical, lebowitz1999gallavotti}:
\begin{equation}\label{eq:e_p}
\dot{S}_{\text{tot}} =  \frac{1}{\tau} \mbox{KL}\left[\mathbf{P}(\{\Z_t\}_{t\in[0, \tau]}),\ \bar{\mathbf{P}}(\{\Z_t\}_{t\in[0, \tau]})  \right],
\end{equation}
where $\tau>0$ is arbitrary and $\{\Z_t\}_{t\in[0, \tau]}$ denotes a stationary path of the process. 
This formulation requires the simulation of the time-reversal dynamics and requires the calculation of the KL divergence between probability measures on path space, which presents computational challenges~\cite{chernyak2006path, chetrite2008fluctuation}.
For SDEs, it is known that the global EPR can be represented in terms of a quadratic form for the time-irreversible drift~\cite{da2023entropy, jiang2004mathematical}. This provides a powerful analytical tool that is particularly accessible for {\ourmethodname}, as its reverse-time dynamics construction yields a closed-form expression. Specifically, $\dot{S}_{\text{tot}}$ for {\ourmethodname} can be expressed as (Methods~\ref{app:Time-reversal and Entropy Production}):
\begin{equation}\label{eq:ep}
\dot{S}_{\text{tot}} = \int_{\R^D} f_{\text{irr}}^{\top} M^{-1}  f_{\text{irr}} \rho(\Z) \d \Z,
\end{equation}
where $f_{\text{irr}}$ is the time-irreversible drift given by~\eqref{eq:irr} and $\rho$ is the stationary density.
Numerically, the {\ourmethodname} is first learned from input trajectories. The learned stochastic dynamics are then simulated over long time horizons to generate a large set of samples that approximate the invariant distribution. The global EPR is subsequently estimated from these samples using Monte Carlo integration of~\eqref{eq:ep}. By combining the above steps, {\ourmethodname} offers a fully data-driven pipeline to estimate irreversibility from raw trajectory data.

$\dot{S}_{\text{tot}}=0$ is necessary and sufficient for equilibrium. 
To spatially quantify the breakdown of time-reversal symmetry and the emergence of non-equilibrium dynamics, we consider the (local) total EPR defined as the integrand of
$\dot{S}_{\text{tot}}$~\cite{pigolotti2017generic, seifert2005entropy, seifert2012stochastic}:
\begin{equation*}
\dot{s}_{\text{tot}} = f_{\text{irr}}^{\top} M^{-1}  f_{\text{irr}}.
\end{equation*}
Beyond characterizing irreversibility, the time-irreversible drift $f_{\text{irr}}$ is identical to the current velocity and can be used to compute other measures such as the system EPR. These connections are discussed in detail in Methods~\ref{app:Time-reversal and Entropy Production}.

\paragraph{Mathematical Characterization}
We now present the main theorem concerning the expressive capacity and uniqueness of {\ourmethodname}, thereby establishing its identifiability from a theoretical standpoint. Detailed proofs can be found in Methods~\ref{app:Mathematical Characterization}.

In the following, we consider the It\^o SDE driven by a standard Brownian motion
\begin{equation}\label{eq:sde}
\dot{\Z}_t= g(\Z_t) + \sigma(\Z_t) \dot{\B}_t,\quad M(\Z) = \sigma(\Z)\sigma^{\top}(\Z)/2,
\end{equation}
where $g$ denotes the drift vector. We will work under the standard assumptions on the drift and diffusion coefficients that ensure that the SDE admits a unique (strong) solution and a unique invariant distribution with density $\rho$~\cite{pavliotis2014stochastic}.

\begin{theorem}\label{the:appro}
For any SDE of the form~\eqref{eq:sde}, there exists a banded antisymmetric matrix $W$ with bandwidth 1 and a potential $V$ satisfying $e^{-V}\in L^1(\R^D)$ such that the drift $g$ admits the following decomposition:
\begin{equation*}\label{eq:hhd}
    g(\Z) = -[M(\Z) + W(\Z)]\nabla V(\Z) + \nabla\cdot M(\Z)  + \nabla\cdot W(\Z),
\end{equation*}
on the support of the stationary density $\{\Z|\rho(\Z)>0\}$.
Furthermore, the potential $V$ is given by $V = -\log\rho$ and is unique up to an additive constant; the resulting potential force $-\nabla V$, the time-reversible drift given by~\eqref{eq:rev}, and the time-irreversible drift given by~\eqref{eq:irr} are uniquely determined.
\end{theorem}

Although existing modeling approaches offer considerable flexibility, {\ourmethodname} is theoretically universal: it can represent any SDE admitting a unique invariant distribution and thus capture the dynamics of any stable and
ergodic dissipative SDE\footnote{An SDE satisfying a formal dissipativity condition with an appropriate Lyapunov function is guaranteed to possess a unique invariant distribution~\cite{huang2015steady}.}.
Here, the functions $\sigma$, $W$ and $V$ can be drawn from any hypothesis class
that satisfies the classical universal-approximation theorem of functions, for example,
feedforward, residual, or convolutional networks. Under this parameterization, the universal-approximation result for the proposed {\ourmethodname} follows directly from Theorem~\ref{the:appro} and the approximation guaranties of the chosen class.

The uniqueness statement in the above theorem provides us with the key identifiable learning capacity of {\ourmethodname}. When applied to a family of SDEs indexed by a parameter, this means that the learned energy landscapes and time-irreversible drifts are guaranteed to be theoretically consistent across all parameter values. This consistency, even with minor numerical errors, provides a robust foundation for meaningfully comparing learned dynamical structures (e.g., energy landscapes, time-irreversible drift) across the parameter space within the SDE family. This will be demonstrated in our numerical examples.
As an important byproduct of our proposed methodology, {\ourmethodname} can directly learn the invariant distribution from trajectory data alone. Although existing methods~\cite{lin2022computing, lin2023computing} are also capable of learning energy-like functions to approximate the invariant distribution, either from trajectory data or from known governing equations, they typically rely on auxiliary loss terms derived from the stationary Fokker-Planck equation to promote consistency. In contrast, our method directly derives the potential function from the structure of the SDE, with a rigorous guarantee that the normalized density $\rho = \frac{1}{\mathcal{Z}} e^{-V}$ is exactly stationary for the learned dynamics.

\subsection{Linear Case}
We begin by validating the effectiveness of {\ourmethodname} in a simple, analytically tractable linear diffusion process.
Consider the 2D linear SDE~\cite{LelievreNierPavliotis2013}:
\[
\dot{\Z}_t = -\bigl(M + W\bigr) S\, \Z_t + \sqrt{2 M}\, \dot{\mathbf{B}}_t,
\]
where the (constant) matrices \(M\) and \(S\) are symmetric positive-definite and \(W\) is antisymmetric with \(W = \lambda W_0\), where $W_0$ is the standard antisymmetric matrix in $2$ dimensions. Linearity yields the closed-form stationary density \(\rho = \mathcal{N}(0, S^{-1})\). 
This SDE coincides exactly with {\ourmethodname} in the linear regime, with the irreversible component
\(-W\, S\, \Z\). 
From~\eqref{eq:ep} we compute the global EPR~\cite{da2023entropy}:
\begin{equation}\label{eq:ep_linear}
\dot{S}_{\text{tot}} = -\mathrm{Tr}\bigl(M^{-1} W S W\bigr).
\end{equation}
While {\ourmethodname} employs non-linear networks and is a non-linear model, the analytical tractability of this linear system  provides exact expressions for the stationary density, the irreversible component, and the EPR, offering a rigorous ground truth for quantitative benchmarking.

\begin{figure}[!ht]
\centering
\includegraphics[width=0.6\linewidth]{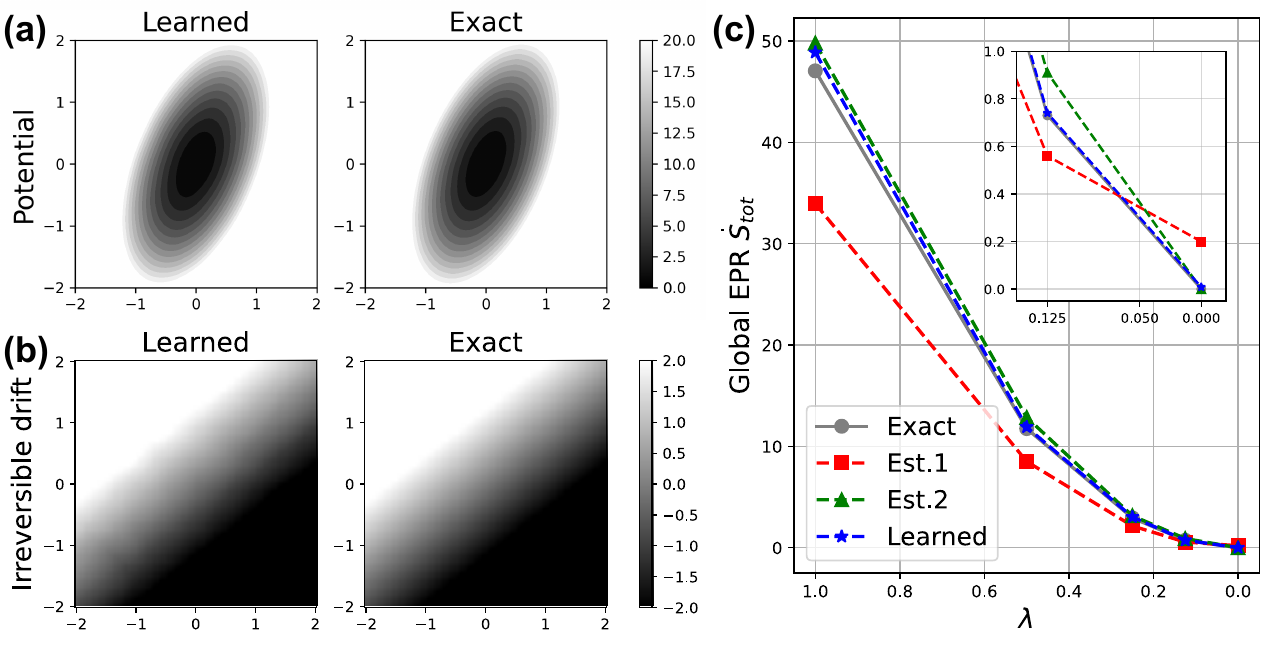}
\caption{Results for linear system.
{\textbf{(a)}} The learned potential and the negative log of the stationary density, each shifted by subtracting their value at input zero, ensuring that both functions attain zero at the origin.
{\textbf{(b)}} First component for the time-irreversible drift of learned and exact system for $\lambda=1$.
{\textbf{(c)}} Comparison of the global EPR calculated via four approaches. The exact value is computed using~\eqref{eq:ep_linear}; Est.$_1$ is calculated by applying a state-space discretization to approximate the definition in~\eqref{eq:e_p};
Est.$_2$ and learned $\dot{S}_{\text{tot}}$ are obtained via Monte Carlo approximation of~\eqref{eq:ep}, using the exact and learned equations, respectively.}
\label{fig:linear}
\end{figure}

We generate trajectory data by simulating this linear SDE and use these data to train {\ourmethodname}. The results presented in Fig.~\ref{fig:linear} provide a direct validation of our approach.
This confirms that our method faithfully reconstructs both the stationary density and the irreversible component of the dynamics. In addition, the global EPR across a range of values of $\lambda$, the parameter that controls the amount of irreversibility, shows excellent agreement with theoretical predictions.
Computing irreversibility from discrete trajectory observations presents significant challenges due to the difficulty in estimating path measures. One common approach involves discretizing the path space (Methods~\ref{app: properties of Global EPR}). However, combined errors arising from both temporal and spatial discretization impair the reliability of the resulting statistical estimates, and the computational cost increases exponentially with the system dimension. As shown in Fig.~\ref{fig:linear} (c), the computation of the EPR via definition (Est.$1$) captures only the qualitative trends of the exact EPR without matching its quantitative values.

\subsection{Polymer Stretching Dynamics}
\makeatletter
\protected@edef\@currentlabelname{Polymer Stretching Dynamics}%
\makeatother
\label{sec:polymer}
We next demonstrate our approach by modeling the temporal evolution
of polymer chain extension under elongational flow.
Owing to its ubiquity in polymer processing, this physical process is fundamental to the study of polymer rheology
and has been a subject of extensive research within
the polymer physics community~\cite{bird1987dynamics,de1974coil,doi1988theory}.
Landmark experiments and simulations at the single-molecule level have revealed a profound heterogeneity, in which identical polymer chains exhibit starkly different stretching dynamics,
and the statistical nature of this heterogeneity is acutely sensitive to system parameters
such as flow strength~\cite{larson2005rheology,perkins1997single,smith1998response}.
A recent deep learning strategy seeks to capture these dynamics by constructing macroscopic descriptors and their effective energy landscape~\cite{chen2024constructing}.
However, the complexity and inherent stochasticity of polymer dynamics lead to identifiability issues in these models.
The learned energy function lacks a unique physical definition, leading to inconsistent solutions even for identical datasets; thus the inferred non-equilibrium components fail to quantitatively capture the system's entropy production.
This ambiguity has made it difficult to systematically investigate the influence of physical parameters on polymer dynamics.
In this subsection, we show that our {\ourmethodname} is able to overcome these limitations, successfully capturing heterogeneous dynamics across various flow strengths while providing a physically meaningful energy landscape and a quantitative measure of entropy production. This capability further reveals scaling behaviors that were unattainable in existing approaches~\cite{chen2024constructing} due to the non-identifiability of potential functions
across different flow strengths.

We simulate the stretching of a single polymer chain in a planar elongational flow for various flow strengths, characterized by the strain rate $\dot{\varepsilon}$ (Fig.~\ref{fig:polymer_V} (a)) .
The chain consists of 300 coarse-grained beads connected by rigid rods, resulting in a system with 900 degrees of freedom before imposing constraints to neglect inertial effects.
The polymer extension, defined as the projected length along the elongational axis, provides a macroscopic view of the microscopic heterogeneity of the polymer population and is therefore chosen as our coordinate of interest \( Z_1 \). Our approach constructs two additional closure coordinates alongside \( Z_1 \) following previous approaches to model reduction~\cite{chen2024constructing}, thereby establishing a three-variable {\ourmethodname} for each value of the strain rate.
Validated on various initial chain conformations, the trained models demonstrate  precise prediction of the chain extension evolution across all cases (Methods~\ref{app: Additional Numerical Results polymer}).

\begin{figure*}[!h]
    \centering
    \includegraphics[width=\linewidth]{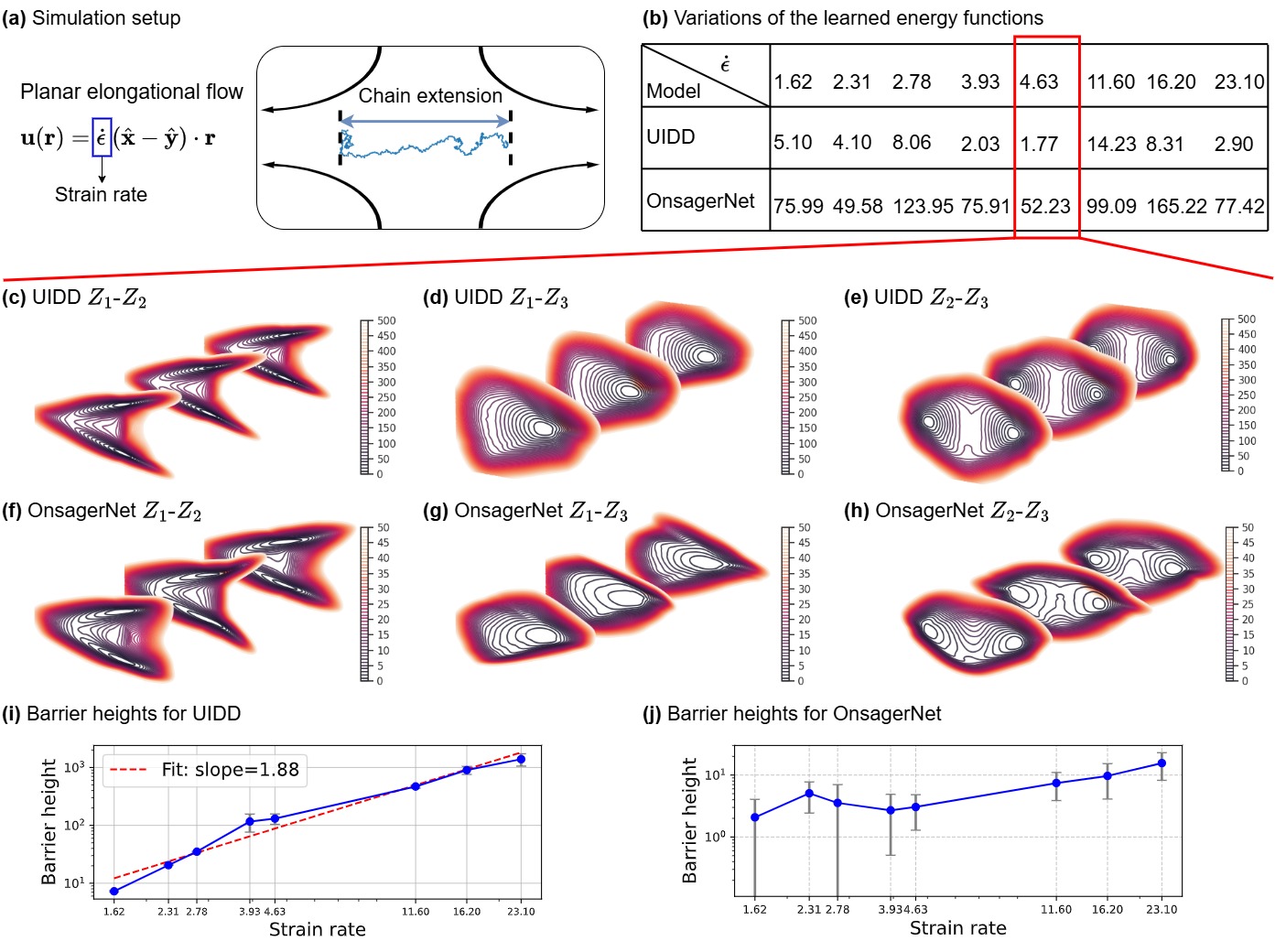}
    \caption{
    Simulation setup, variations across random seeds and barrier heights of the learned potential energy for polymer stretching dynamics,
    (\textbf{a}) Linear bead-rod polymer chains are simulated using a Brownian dynamics framework under a planar elongational flow, whose strength is characterized by the strain rate $\dot{\varepsilon}$.
    (\textbf{b}) Relative variation ($\times 10^3$) in learned potential forces over $4$ random seeds at different strain rate $\dot{\varepsilon}$. (\textbf{c-h}) Projections of the learned potential energy at $\dot{\varepsilon}=4.63$ onto the $Z_1\text{-}Z_2$ (\textbf{c,f}), $Z_1\text{-}Z_3$ (\textbf{d,g}), and $Z_2\text{-}Z_3$ (\textbf{e,h}) planes,
    comparing {\ourmethodname} (\textbf{c-e}) and OnsagerNet (\textbf{f-h}) across $3$ seeds. Here the projections are obtained via minimization (e.g., $V(Z_1, Z_2) = \min_{Z_3} V(Z_1, Z_2, Z_3)$), which closely approximates the marginal potential at low temperatures. (\textbf{i.j}) Energy landscape barrier heights of {\ourmethodname} and OnsagerNet averaged over $4$ independent experiments; error bars indicate one standard deviation.
}
    \label{fig:polymer_V}
\end{figure*}

\begin{figure*}[!t]
    \centering
    \includegraphics[width=1\linewidth]{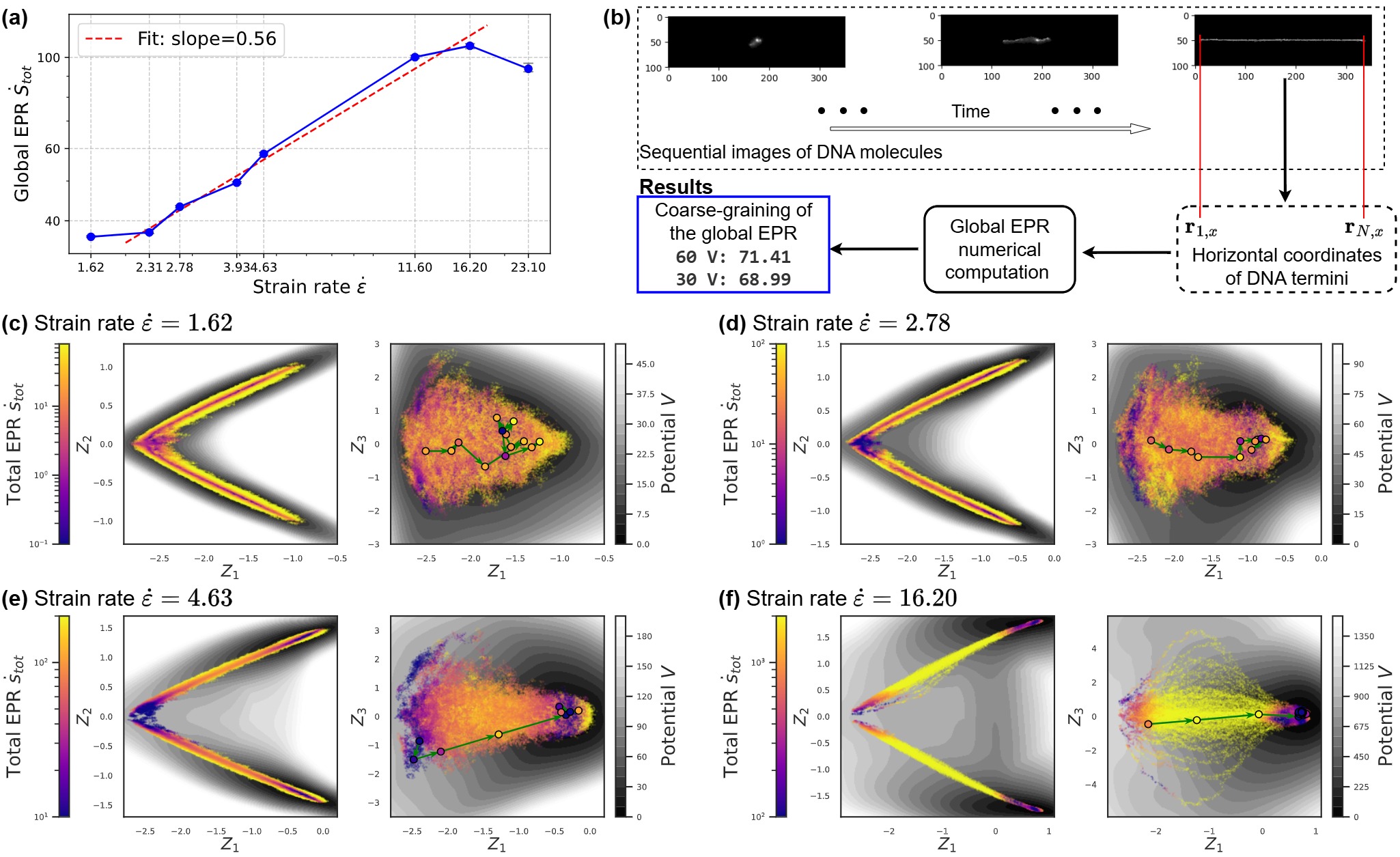}
    \caption{EPR for polymer stretching dynamics. (\textbf{a}) The global EPR obtained via {\ourmethodname} for varying strain rate $\dot{\varepsilon}$, averaged over $4$ independent experiments with error bars showing one standard deviation.
    (\textbf{b}) Illustration of the experimental validation and coarse-graining of the global EPR based on device-level observations.
    (\textbf{c-f}) The total EPR  $\dot{s}_{\text{tot}}$ on test trajectories projected onto the $Z_1\text{-}Z_2$ and $Z_1\text{-}Z_3$ planes. The background is the corresponding potential energy surface. One example trajectory from the test dataset are overlaid on the $Z_1\text{-}Z_3$ panels.}
    \label{fig:polymer_epr}
\end{figure*}

Figure~\ref{fig:polymer_V} provides a numerical validation of the identifiability of the proposed {\ourmethodname}. Specifically, Panel (b) demonstrates that the relative variation in the learned potential forces across different random seeds is significantly smaller for our method compared to standard OnsagerNet.
Panels (c-h) further illustrate this by showing that {\ourmethodname} consistently produces identical projected potential surfaces at $\dot{\varepsilon}=4.63$, irrespective of random initialization. 

Having established the identifiability, we now investigate the energy landscape of the polymer stretching dynamics under a range of strain rates $\dot{\varepsilon}$. 
The barrier height provides a quantitative measure of the energetic difficulty associated with transitions between metastable configurations~\cite{hanggi1990reaction,van1992stochastic}. In the context of polymer stretching, it represents the effective free-energy barrier associated with the flipping of chain ends against the imposed elongational flow. Although it does not exactly correspond to the transition rates in our nonequilibrium setting, it serves as a qualitative descriptor of the effective energy landscape.
Since {\ourmethodname} directly provides the logarithm of the stationary density,
it allows for a straightforward computation of the barrier height (Methods~\ref{app:polymer bh}).
As shown in Fig.~\ref{fig:polymer_V} (i),
the barrier height associated with polymer stretching exhibits a super-linear increase with
the strain rate $\dot{\varepsilon}$.
In contrast, the OnsagerNet approach fails to produce consistent energy functions for fixed parameters, resulting in barrier heights that show no clear or systematic dependence on the strain rate, as illustrated in Fig.~\ref{fig:polymer_V} (j).
  
Next, we examine the irreversibility by computing the global EPR,
$\dot{S}_{\text{tot}}$, a key marker of non-equilibrium. The log-log plot in Fig.~\ref{fig:polymer_epr} (a) reveals a clear scaling relationship with the strain rate $\dot{\varepsilon}$. Specifically, $\dot{S}_{\text{tot}}$ exhibits sub-linear growth, tending towards a plateau in both the weak- and strong-driving regimes. The fact that $\dot{S}_{\text{tot}}$ is consistently positive for all strain rates provides evidence that the polymer system operates far from equilibrium.

To understand the origin of this global irreversibility, we subsequently
measure the total EPR. Fig.~\ref{fig:polymer_epr} (c-f) shows distinct dynamical regimes
governed by the strain rate $\dot{\varepsilon}$ (see Methods~\ref{app: Additional Numerical Results polymer} for results at other $\dot{\varepsilon}$ values).
At lower magnitudes of $\dot{\varepsilon}$, the total EPR exhibits minimal temporal variation along the individual trajectories of the polymers. In contrast, for larger values of $\dot{\varepsilon}$,  the local EPR displays pronounced spatiotemporal heterogeneity during the unfolding process.  Specifically, the total EPR is significantly suppressed at both the initial state (folded polymer) and the final state (fully extended conformation),
yet it peaks dramatically during the mid-unfolding phase.
This strain rate-dependent bimodal distribution of total EPR highlights the role of
non-equilibrium dynamics in driven polymer unfolding,
where maximum irreversibility occurs during mid-trajectory unfolding events.

Finally, we conduct single-molecule real-world experiments to verify the qualitative results that the polymer system operates far from equilibrium. A property of the global EPR is that the EPR computed from projected dynamics is a lower bound for that of the full system (Methods~\ref{app: properties of Global EPR}). Leveraging this property, we can assess non-equilibrium behavior directly from device-level observations by evaluating the EPR of the measured coordinates.
We track the stretching dynamics of individual DNA molecules subjected to a planar elongational field, where the strain rate is regulated by the applied voltages (Methods~\ref{app: exp polymer data}). Configurations in which the DNA chain is fully extended are selected to ensure that the system had reached a steady state. From the image-recorded configurations, we extract the time series of the horizontal coordinates of the left and right endpoints, corresponding to the atomic positions at both termini of the DNA molecule. Using only these two-dimensional trajectory data, we numerically compute the global EPR based on the definition (\ref{eq:e_p}). This procedure is illustrated in Fig.~\ref{fig:polymer_epr} (b).
The resulting EPR values, $71.41$ at $60 V$ and $68.99$ at $30 V$, remain strictly positive even with partial information, consistent with the non-equilibrium behavior predicted by {\ourmethodname}. These findings highlight the validity of our approach in extracting physically meaningful insights from real polymer dynamics.

\subsection{Stochastic Gradient Langevin Dynamics}
\makeatletter
\protected@edef\@currentlabelname{Stochastic Gradient Langevin Dynamics}%
\makeatother
\label{sec:sgld}
Stochastic gradient langevin dynamics (SGLD)~\cite{welling2011bayesian} is a cornerstone algorithm in modern machine learning, 
bridging stochastic optimization and Bayesian sampling. 
It extends standard mini-batch gradient descent by injecting isotropic Gaussian noise to sample from the target distribution:
\[
p(\mathbf Z) \propto e^{-L(\mathbf Z)}, \qquad  
L(\mathbf Z) = \sum_{i=1}^{n} L_i(\mathbf Z),
\]
where $L_i(\mathbf Z)$ denotes the loss (or negative log-likelihood) associated with the $i$-th data sample. 
The discrete-time dynamics are given by
\[
\mathbf Z^{(t+1)} 
= \mathbf Z^{(t)} 
- \eta \frac{n}{b} \sum_{i=1}^{b} \nabla L_{\gamma_i}(\mathbf Z^{(t)}) 
+ \sqrt{2\eta}\,\boldsymbol{\xi}^{(t)},
\]
where $\{\gamma_i\}_{i=1}^{b}$ indexes a mini-batch of size $b$, 
$\eta>0$ is the step size, 
and $\boldsymbol{\xi}^{(t)} \sim \mathcal N(\mathbf 0,\mathbf I)$ is standard Gaussian noise.

\begin{figure}[!ht]
\centering
    \includegraphics[width=0.6\linewidth]{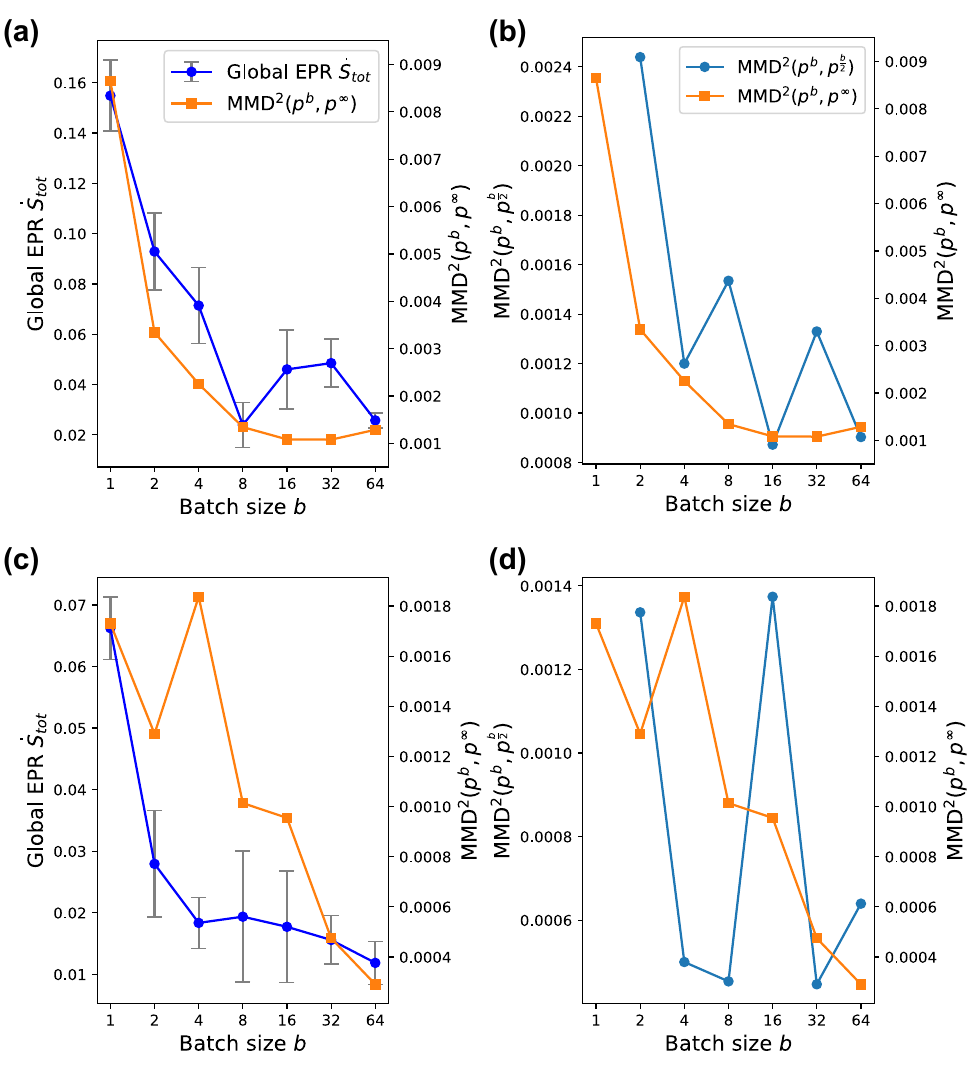}
    \caption{The global EPR and the squared MMD for SGLD with varying mini-batch sizes. Results represent the mean of $4$ independent experiments, with error bars denoting one standard deviation. Two case studies are considered: (\textbf{a,b}) least squares regression and (\textbf{c,d}) independent component analysis. In all panels, the squared MMD between the mini-batch ($p^b$) and full-batch ($p^\infty$) invariant distributions of the SGLD sampler is plotted as a reference. Panels (\textbf{a}) and (\textbf{c}) further display the global EPR, while panels (\textbf{b}) and (\textbf{d}) show the squared MMD between consecutive batch sizes ($p^b$ and $p^{b/2}$), illustrating how both measures vary with increasing batch sizes.
    }
    \label{fig:sgld_ep}
\end{figure}

While the injected noise is designed to drive the system toward an equilibrium Gibbs-Boltzmann distribution, the stochasticity inherent in mini-batch gradient approximation complicates this picture. In the continuous-time limit, the mini-batching introduces a state-dependent covariance into the diffusion term~\cite{li2019stochastic}:
\[
\dot{\Z} = -\nabla L(\Z) + (\eta\,\Sigma(\Z) + 2)^{\frac{1}{2}} \,\dot{\B}_t.
\]
This state-dependent diffusion breaks detailed balance. The resulting irreversibility drives the sampler to a biased non-equilibrium steady state instead of the target equilibrium distribution. Consequently, the irreversibility becomes a direct, physical diagnostic for the sampling bias, providing a quantitative link between the choice of mini-batch size and the trustworthiness of SGLD. In contrast, standard statistical metrics, such as Maximum Mean Discrepancy (MMD) or KL divergence, are often infeasible for this purpose as they require a priori knowledge of the ground-truth distribution.
Our data-driven methodology, which enables us to detect irreversibility directly from the output trajectories of the sampling scheme, offers a promising tool for quantifying this bias.
In other words, by applying our method to detect the degree of non-reversibility,
we can judiciously choose a mini-batch size $b$ that balance computational efficiency and sampling bias.

We now apply our method to analyze SGLD in the context of least squares regression (linear) and independent components analysis (nonlinear) detailed in Methods~\ref{app:sgld} to measure the irreversibility and validate the connection between irreversibility and sampling bias.
Fig.~\ref{fig:sgld_ep} quantifies the impact of mini-batch size on SGLD irreversibility,
measured by the global EPR, across varying batch sizes.
Across two distinct systems, the global EPR consistently decreases with increasing mini-batch size
at small scales, and plateaus at larger scales due to numerical precision constraints.
These results indicate that increasing mini-batch size progressively suppresses non-equilibrium
dynamics in SGLD. At sufficiently large batch sizes,
negligible global total EPR values indicate that the SGLD performs near detailed balance, and that the non-equilibrium bias vanishes.

Furthermore, we compute the MMD between the invariant distributions from different batch sizes
and that of full-batch SGLD. As shown in Fig.~\ref{fig:sgld_ep} (a, c),
the MMD trend as a function of batch size is highly consistent with that of the
EPR (see Methods~\ref{app:sgld}~for theoretical justification).
For comparison, we test a heuristic convergence metric based on the MMD between
distributions from consecutive mini-batch sizes (i.e., batch sizes $b$ and $b/2$).
This baseline metric can be misleading, showing trends inconsistent with actual convergence or lagging as an indicator, as shown in Fig.~\ref{fig:sgld_ep} (b, d).
These findings demonstrate that {\ourmethodname} constructed using only sampling trajectories
can effectively compute the EPR as a physically meaningful diagnostic of sampling bias.

\section{Discussion}

This work introduces {\ourmethodname}, a framework that unifies physical interpretability, universal approximation, and rigorous identifiability for learning complex dissipative dynamics. It ensures a unique energy landscape and a clear decomposition of the drift force into reversible and irreversible components. This structure enables key thermodynamic quantities, particularly the entropy production rate that characterizes time-irreversibility, to be directly inferred from data.
Applications to polymer stretching and stochastic gradient Langevin dynamics demonstrate the versatility of {\ourmethodname} as a general tool for uncovering unknown dynamics in non-equilibrium phenomena. With its identifiable and physically consistent representation of dynamics, the framework opens new avenues for data-driven discovery and control of complex systems, offering a unified perspective on diverse nonequilibrium processes across scientific disciplines~\cite{dobson2003protein,whitesides2002self,zhao2023learning}. Building on these foundations, several promising directions remain for further exploration. A natural extension is to apply our approach to learn the non-equilibrium dynamics of other scientifically important systems, such as information processing in the brain~\cite{nartallo2026nonequilibrium}, the thermodynamics of active matter~\cite{marchetti2013hydrodynamics}, and protein folding~\cite{dobson2003protein}.
Another direction is to design equilibrium-preserving sampler and numerical difference scheme by systematically quantifying and reducing irreversibility. Finally, our current framework inherently focuses on macroscopic dynamics learned from coarse-grained descriptions. Extending it to learn directly from mesoscopic or microscopic information through an operator-level formulation offers an exciting pathway toward broader and more versatile applications.

\section*{Data and Code Availability}
The simulation and experimental datasets produced and analyzed for \nameref{sec:polymer} are publicly available at the following Hugging Face repositories: \url{https://huggingface.co/datasets/MLDS-NUS/Fs_T_Reduced}, \url{https://huggingface.co/datasets/MLDS-NUS/Fs_Validation_Reduced}, and \url{https://huggingface.co/datasets/MLDS-NUS/Experimental_Images}, corresponding respectively to the primary simulation trajectories, validation simulation data, and validation experimental images. Further details on data generation are provided in Methods~\ref{app:Supplementary Information for Numerical Results}.

The code used to reproduce the results is publicly available on GitHub at \url{https://github.com/MLDS-NUS/UIDD-jax}.

\section*{Acknowledgements}
This project is supported by the National Research Foundation, Singapore, under its AI Singapore Programme (AISG Award No.: AISG3-RP-2022-028).
Q.L. and A.Z. are partially supported by the National Research Foundation, Singapore, under the NRF fellowship (Project No. NRF-NRFF13-2021-0005). B.W.S. is supported by the National University of Singapore Presidential Young Professorship start-up grant (Grant No. A-0010260-00-00).
G.A.P. is partially supported by an ERC-EPSRC Frontier Research Guarantee through Grant No. EP/X038645, ERC Advanced Grant No. 247031 and a Leverhulme Trust Senior Research Fellowship, SRF$\backslash$R1$\backslash$241055.

\bibliographystyle{abbrv}
\bibliography{ref}

\appendix

\section*{Methods}\label{sec:appendix}
 
\section{Comparison to Existing Dissipative Dynamics Models}\label{app:Existing Models}
We compare {\ourmethodname} with several thermodynamically motivated frameworks for learning dissipative dynamics.

\paragraph{Stat-PINNs}
When compared with Stat-PINNs~\cite{huang_al_2025}, {\ourmethodname} can be viewed as their non-equilibrium extension. Stat-PINNs correspond precisely to the time-reversible (symmetric) sector of {\ourmethodname}, recovering only the gradient-driven component of the dynamics. By additionally learning the antisymmetric operator $W$ and its divergence, {\ourmethodname} captures the full irreversible drift and can therefore represent general non-equilibrium behavior.
We note that directly adding a general antisymmetric operator $W$
(e.g. in OnsagerNet~\cite{yu2021onsagernet}) leads to non-identifiability,
hence a careful construction of the banded structure of $W$ 
in {\ourmethodname} is required.

\paragraph{S-OnsagerNet}
Compared with S-OnsagerNet~\cite{chen2024constructing}, which also incorporates non-equilibrium behavior, {\ourmethodname} employs specific matrix functions for both $M$ and $W$ and includes their divergences as separate additive terms. These structures are not ad-hoc modifications but arise from a principled synthesis of the generalized Onsager principle~\cite{yu2021onsagernet}, the generalized Helmholtz decomposition~\cite{da2023entropy}, and the essentially Hamiltonian decomposition~\cite{feng1995volume}, see Methods~\ref{app:Mathematical Characterization}. The resulting dynamical model provides an identifiable extension of S-OnsagerNet that, in addition, allows computation of both the invariant distribution and the EPR, as demonstrated in the applications.

\paragraph{Deterministic Models}
{\ourmethodname} naturally incorporates finite-temperature conditions, where temperature controls the noise intensity and influences thermodynamic quantities. To make this explicit, we introduce a parameter $\T$ representing temperature and rescale
\[
\sigma = \sqrt{\T}\,\tilde{\sigma}, \qquad
M = \T \tilde{M},
\]
together with the re-normalizations
\[
V = \tilde{V}/\T, \qquad
W = \T \tilde{W}.
\]
Substituting these rescaled quantities into the original formulation yields the $\T$-dependent form
\begin{equation*}
\dot{\Z}_t
= -[\tilde{M} + \tilde{W}] \nabla \tilde{V}
  + \T \nabla\cdot \tilde{M}
  + \T \nabla\cdot \tilde{W}
  + \sqrt{\T}\,\tilde{\sigma}\,\dot{\B}_t.
\end{equation*}
Note that $\tilde{W}$ and $\tilde{V}$ still depend on $\T$. Taking the limit $\T \to 0$ and assuming that $\tilde{W}$ and $\tilde{V}$ have well-defined limits, {\ourmethodname} formally reduces to the deterministic OnsagerNet dynamics~\cite{yu2021onsagernet},
\[
\dot{\Z}_t = -[M(\Z_t) + W(\Z_t)] \nabla V(\Z_t),
\]
as both the divergence terms and the noise vanish. In this limit, the formulation also coincides with the GENERIC formalism~\cite{lee2021machine, zhang2022gfinns},
\[
\dot{\Z}_t
= L(\Z_t)\,\nabla E(\Z_t)
  + M(\Z_t)\,\nabla S(\Z_t),
\quad
L(\Z_t)\,\nabla V(\Z_t)
= M(\Z_t)\,\nabla E(\Z_t)
= 0,
\]
which can be written in Onsager form by taking $V = -(E + S)$ and $W = L$. It is noted that these deterministic models, and even models that merely add noise directly, are all non-identifiable. The key ingredients of UIDD, including the noise component that fixes $M$, the banded structure imposed on $W$, and the added divergence term, are all all essential for identifiability.

\paragraph{Classical Dissipative Formulations}
Beyond neural-network-based approaches, there is a long-standing literature on thermodynamically consistent formulations of dissipative dynamics. 
Classical examples include Langevin equations~\cite{ayala2025reversibility}, the Onsager principle~\cite{doi2011onsager,onsager1931reciprocal}, the GENERIC formalism, and related potential-based formulations~\cite{ao2004potential}.
These classical models specify the dissipative and conservative components of the dynamics in terms of prescribed energy, entropy, and transport operators. {\ourmethodname} is designed to be compatible with this viewpoint: for fixed parameterizations of $V$, $M$, and $W$, it reduces to a standard SDE with well-defined thermodynamic structure, while its neural representation makes these objects learnable directly from data.

\paragraph{Overall Comparison}
In summary, Stat-PINNs~\cite{huang_al_2025} correspond to the equilibrium, time-reversible sector with identifiability; S-OnsagerNet~\cite{chen2024constructing} extends this picture to stochastic dynamics without enforcing identifiability; and OnsagerNet or GENERIC based models~\cite{yu2021onsagernet,lee2021machine,zhang2022gfinns} arise as deterministic, zero-temperature limits. {\ourmethodname} simultaneously extends these approaches to finite-temperature, stochastic, and non-equilibrium settings, yielding a mathematically universal and identifiable framework.

\section{Theoretical Results}\label{app: the}
In this subsection, we present the theoretical results concerning the {\ourmethodname}s.
We consider an It\^o SDE driven by standard Brownian motion
\begin{equation}\label{eq:sdeapp}
\dot{\Z}_t= g(\Z_t) + \sigma(\Z_t) \dot{\B}_t,\quad M(\Z) = \sigma(\Z)\sigma^{\top}(\Z)/2,
\end{equation}
where $g$ and $\sigma$ denote the drift and diffusion coefficients, respectively.
To ensure the well-posedness of~\eqref{eq:sdeapp} and the classical validity of its stationary Fokker-Planck equation,
we adopt the following standard technical assumptions:
\begin{assumption}\label{ass:app1}
1. Linear growth: There exists a constant \(K_g > 0\) such that for all \(\Z \in \mathbb{R}^D\),
    \[
    \|g(\Z)\| + \|\sigma(\Z)\| \leq K_g (1 + \|\Z\|).
    \]
2. Lipschitz continuity: There exists a constant \(K_L > 0\) such that for all \(\Z, \Z' \in \mathbb{R}^D\),
    \[
    \|g(\Z) - g(\Z')\| + \|\sigma(\Z) - \sigma(\Z')\| \leq K_L \|\Z - \Z'\|.
    \]
\end{assumption}
\noindent Under these conditions,~\eqref{eq:sdeapp} admits a unique strong solution~\cite{prevot2007concise}.

\begin{assumption}\label{ass:app2}
1. Coefficient regularity: \(g \in C^1(\mathbb{R}^D; \mathbb{R}^D)\), \(\sigma \in C^2(\mathbb{R}^D; \mathbb{R}^{D \times D})\).
2. Invariant distribution regularity: The SDE admits a unique invariant distribution with a probability density function \(\rho: \mathbb{R}^D \to \mathbb{R}\). Furthermore, \(\rho \in L^1(\mathbb{R}^D) \cap C^2(\mathbb{R}^D)\).
\end{assumption}
\noindent These conditions ensure that the stationary Fokker-Planck equation holds in the classical (PDEs) sense.

For convenience, we repeat the dynamics of {\ourmethodname}:
\begin{equation}\label{eq: app_OnsagerHHD}
\begin{aligned}
\dot{\Z}_t=& -[M(\Z_t) + W(\Z_t)]\nabla V(\Z_t) + \nabla\cdot M(\Z_t)  + \nabla\cdot W(\Z_t) + \sigma(\Z_t) \dot{\B}_t.
\end{aligned}
\end{equation}
Here $V(\cdot)$ is a scalar potential, and $M(\cdot)$ and $W(\cdot)$ are $D \times D$ matrix-valued functions. $M(\Z) = \sigma(\Z)\sigma^{\top}(\Z)/2$ is symmetric positive definite, while $W(\Z)$ is a banded antisymmetric matrix whose nonzero entries are confined to the first sub-diagonal and its corresponding counterparts:
\begin{equation}\label{eq:app_W}
W(\Z) \!=\!\!\begin{pmatrix}
    0& H_1(\Z)& \cdots&0&0\\
    -H_1(\Z)&0& \cdots&0&0\\
    0&-H_2(\Z)& \cdots&0&0\\
    \vdots&\vdots & \ddots&\vdots&\vdots\\
    0&0& \cdots&0&H_{D-1}(\Z)\\
    0&0& \cdots&-H_{D-1}(\Z)&0\\
\end{pmatrix}.
\end{equation}
Moreover, we assume that this system satisfies assumptions~\ref{ass:app1} and~\ref{ass:app2}. We set
\begin{equation}\label{eq:app_h2of}
f(\Z) = -[M(\Z) + W(\Z)]\nabla V(\Z) + \nabla\cdot M(\Z) + \nabla\cdot W(\Z).
\end{equation}
%
%
\subsection{Stationary Density}\label{app: stationary density}
We verify that the density
\begin{equation}\label{eq:app_s_d}
\rho(\Z) = \frac{1}{\mathcal{Z}} {e^{-V(\Z)}}, \quad \mathcal{Z} = \norm{e^{-V}}_{L^1}.
\end{equation}
solves the stationary Fokker-Planck equation:
\begin{equation}\label{eq:sfpe}
\nabla \cdot (f\rho - \nabla \cdot (M \rho))=0.
\end{equation}
Substituting~\eqref{eq:app_s_d}, we obtain:
\begin{equation}\label{eq:appmrho}
\nabla \cdot (M \rho) = \rho [-M(\Z) \nabla V(\Z) + \nabla\cdot M(\Z)],
\end{equation}
and consequently by~\eqref{eq:app_h2of} we have
\[
f\rho - \nabla \cdot (M \rho) = \rho [-W(\Z) \nabla V(\Z) + \nabla\cdot W(\Z)] = \nabla \cdot (W \rho).
\]
For each index $d=1, \cdots, D-1$, let $J_d$ denote the $D \times D$ anti-symmetric coupling matrix whose only non-zero entries are $+1$ at $(d, d+1)$ and $-1$ at $(d+1, d)$. Then given the form of matrix $W$, we have
\[
\nabla\cdot \nabla \cdot (W\rho)=\sum_{d=1}^{D-1} \nabla\cdot (J_d \nabla (H_d \rho)) = 0,
\]
which concludes that $\rho$ is the stationary density.

\subsection{Time-reversal and Entropy Production}\label{app:Time-reversal and Entropy Production}
The theory of time-reversed diffusions was originally developed by Hausmann and
Pardoux~\cite{haussmann1986time}.
It was subsequently extended by Millet, Nualart,
and Sanz under modified conditions~\cite{millet1989integration}.
We adopt the framework of Da Costa and Pavliotis~\cite{da2023entropy},
which applies to SDEs with locally Lipschitz coefficients.
\begin{lemma}
Assume that the solution $(\Z_t)_{t\in[0,T]}$ is stationary with respect to the
density $\rho$. Then, its time-reversed process $\bar{\Z}_t$ is a Markov diffusion process,
with diffusion $\bar{\sigma}=\sigma$ and drift
\begin{equation*}
\bar{f}(\Z) = \left\{\begin{aligned}
    &-f(\Z)+ 2\rho^{-1} \nabla \cdot (M\rho)(\Z),\quad &&\text{when } \rho(\Z)>0,\\
    & -f(\Z)\quad &&\text{when } \rho(\Z)=0.
\end{aligned} \right.
\end{equation*}
\end{lemma}
\noindent With the above lemma, we now able to derive the time-reversed diffusion process for the {\ourmethodname}.
Substituting~\eqref{eq:app_h2of},~\eqref{eq:app_s_d}, and~\eqref{eq:appmrho}, we obtain:
\[
-f(\Z)+ 2\rho^{-1} \nabla \cdot (M\rho)(\Z) =
-[M(\Z) - W(\Z)]\nabla V(\Z) + \nabla\cdot M(\Z) - \nabla\cdot W(\Z),
\]
which concludes the proof.

\paragraph{Global EPR}
By the time-reversed diffusions, we can calculate the global EPR.
A detailed derivation is available in~\cite{jiang2004mathematical},
and its extension to cases with a noninvertible diffusion matrix
is treated in~\cite{da2023entropy}.
To make this SI self-contained, we briefly recount the derivation steps.

We first assume that the integral is finite, i.e.,
\begin{equation*}
\int_{\mathbb{R}^D} f_{\text{irr}}^{\top} M^{-1}  f_{\text{irr}} \rho(\Z) d\Z < \infty,
\end{equation*}
which ensures the conditions of Girsanov's theorem are satisfied. Noting that the drift difference between the forward and reverse processes equals \(2f_{\text{irr}}\), we obtain the Radon-Nikodym derivative for their path measures:
\begin{equation*}
\frac{\d \bar{\mathbf{P}}(\{\Z_t\}_{t\in [\tau, \tau+\varepsilon]} | \Z_0)}{\d \mathbf{P}(\{\Z_t\}_{t\in [\tau, \tau+\varepsilon]} | \Z_0)} = \exp\left( - \int_{\tau}^{\tau+\varepsilon}  f_{\text{irr}}^{\top} M^{-1}  f_{\text{irr}} dt + \int_{\tau}^{\tau+\varepsilon} f_{\text{irr}}^{\top} M^{-1} (d\Z_t - f(\Z_t) dt) \right).
\end{equation*}
Assuming both forward- and reverse-time SDE are solved with initial conditions drawn from density \(\rho\), we have
\begin{align*}
\frac{1}{\varepsilon} \mbox{KL} \left[\mathbf{P}_{[\tau, \tau+\varepsilon]},\ \bar{\mathbf{P}}_{[\tau, \tau+\varepsilon]}  \right]&= \frac{1}{\varepsilon} \left\langle \log \frac{\d \mathbf{P}(\{\Z_t\}_{t\in [\tau, \tau+\varepsilon]})}{\d \bar{\mathbf{P}}(\{\Z_t\}_{t\in [\tau, \tau+\varepsilon]})} \right\rangle \\
&= -\frac{1}{\varepsilon} \left\langle \log \frac{\d \bar{\mathbf{P}}(\{\Z_t\}_{t\in [\tau, \tau+\varepsilon]} | \Z_0) \rho(\Z_0)}{\d \mathbf{P}(\{\Z_t\}_{t\in [\tau, \tau+\varepsilon]} | \Z_0)\rho(\Z_0)} \right\rangle \\
&= \frac{1}{\varepsilon} \left\langle \int_{\tau}^{\tau+\varepsilon} f_{\text{irr}}^{\top} M^{-1}  f_{\text{irr}} dt \right\rangle - \frac{\sqrt{2}}{\varepsilon} \left\langle \int_{\tau}^{\tau+\varepsilon} f_{\text{irr}}^{\top}(\Z_t) M^{-1/2} dW_t \right\rangle \\
&= \frac{1}{\varepsilon} \left\langle \int_{\tau}^{\tau+\varepsilon} f_{\text{irr}}^{\top} M^{-1}  f_{\text{irr}} dt \right\rangle \\
&= \int_{\mathbb{R}^D} f_{\text{irr}}^{\top} M^{-1}  f_{\text{irr}} \rho(\Z) d\Z.
\end{align*}

\paragraph{Current Velocity and System EPR}
The time-irreversible drift plays a fundamental role in the dynamics of the system.
Besides being used in computing the global EPR,
$f_{\text{irr}}$ is identical to the current velocity~\cite{nelson1967dynamical, pavliotis2014stochastic, seifert2005entropy},
as shown by the equivalence:
\[
f(\Z) - \rho^{-1} \nabla \cdot (M\rho)(\Z) = -W(\Z)\nabla V(\Z) +\nabla\cdot W(\Z) = f_{\text{irr}}.
\]
Therefore, the ODE governed by $f_{\text{irr}}$, i.e.,
\begin{equation}\label{eq:app_irr}
    \dot{\Z}_t^{\text{irr}}= f_{\text{irr}}(\Z_t^{\text{irr}}) = - W(\Z_t^{\text{irr}})\nabla V(\Z_t^{\text{irr}}) + \nabla\cdot W(\Z_t^{\text{irr}}),
\end{equation}
defines the characteristics associated with the stationary FPE given by~\eqref{eq:sfpe}. The resulting flow map $\phi_t$, defined via $\phi_t(\Z) = \Z_t^{\text{irr}}$ where $\Z_t^{\text{irr}}$ solves~\eqref{eq:app_irr} with initial condition $\Z_0^{\text{irr}}=\Z$, preserves the stationary density. This implies that for any observable $A(\Z)$, the following identity holds:
\begin{equation*}
\forall t \in \R, \quad\int_{\R^D} A(\phi_t(\Z))\rho(\Z) \d \Z = \int_{\R^D} A(\Z)\rho(\Z) \d \Z.
\end{equation*}

We then show that the system EPR can be computed from $f_{\text{irr}}$.
The system EPR quantifies the trajectory-wise evolution of the stochastic Gibbs entropy, defined as~\cite{seifert2005entropy, seifert2012stochastic}:
\begin{equation}\label{eq:ep_sys}
\dot{s}_{\text{sys}} = -f_{\text{irr}} \cdot \nabla \log \rho.
\end{equation}
This measure characterizes the internal entropy evolution of the system.
At stationarity, the ensemble average vanishes:
$\int_{\mathbb{R}^D} \dot{s}_{\text{sys}} \rho(\Z) \d \Z = 0$.

Substituting the stationary density \(\rho \) from~\eqref{eq:app_s_d}, we obtain:
\begin{equation*}
\dot{s}_{\text{sys}} = f_{\text{irr}} \cdot \nabla \log \rho = -f_{\text{irr}} \cdot \nabla V.
\end{equation*}
Recalling that \(f_{\text{irr}} = -W \nabla V + \nabla \cdot W \) and that \(W\) is anti-symmetric, we derive:
\begin{equation*}
-f_{\text{irr}} \cdot \nabla V = (\nabla V)^{\top} W \nabla V - (\nabla \cdot W) \cdot \nabla V = - (\nabla \cdot W) \cdot \nabla V,
\end{equation*}
where the term \((\nabla V)^{\top} W \nabla V = 0\) due to the anti-symmetry of \(W\).
Then, use the property of divergence we have
\begin{equation*}
\nabla \cdot (W \nabla V) = (\nabla \cdot W) \cdot \nabla V + \operatorname{Tr}(W \nabla^2 V) = (\nabla \cdot W) \cdot \nabla V,
\end{equation*}
where $\operatorname{Tr}(W \nabla^2 V)=0$ since $W$ is anti-symmetric and $ \nabla^2 V$ is symmetric.
Finally, we compute the divergence of \(f_{\text{irr}}\):
\begin{equation*}
\nabla \cdot f_{\text{irr}}
= -\nabla \cdot (W \nabla V) + \nabla \cdot (\nabla \cdot W) \nonumber  = -\nabla \cdot (W \nabla V),
\end{equation*}
where \(\nabla \cdot (\nabla \cdot W) = 0\) because the second divergence of a anti-symmetric matrix field vanishes. In summary, for {\ourmethodname}, the system EPR has several equivalent expressions:
\begin{equation*}
\dot{s}_{\text{sys}} = -f_{\text{irr}} \cdot \nabla V = -(\nabla\cdot W) \cdot \nabla V = -\nabla \cdot(W\nabla V) = \nabla\cdot f_{\text{irr}}.
\end{equation*}

\subsection{Mathematical Characterization}\label{app:Mathematical Characterization}
In this subsection, we prove the expressivity and uniqueness of {\ourmethodname}.
Our arguments mainly rely on the Helmholtz decomposition~\cite{da2023entropy,
eyink1996hydrodynamics,graham1977covariant, jiang2004mathematical,
pavliotis2014stochastic} and the essentially Hamiltonian decomposition~\cite{feng1995volume}. The main theorem is repeated below. 
Recall that $M = \sigma\sigma^\top /2$.
\begin{theorem}
For any SDE as given by~\eqref{eq:sdeapp}, there exists a banded anti-symmetric matrix $W$ with bandwidth 1 and a potential $V(\Z)$ satisfying $e^{-V}\in L^1(\R^D)$ such that  the drift $g$ admits the following decomposition:
\begin{equation*}
    g(\Z) = -[M(\Z) + W(\Z)]\nabla V(\Z) + \nabla\cdot M(\Z)  + \nabla\cdot W(\Z),
\end{equation*}
on the support of the stationary density $\{\Z|\rho(\Z)>0\}$.
Furthermore, the potential $V$ is given by $V = -\log\rho$ and is unique, up to an additive constant; the resulting potential force $\nabla V$, the time-reversible drift $f_{\text{rev}} = -M(\Z)\nabla V(\Z) +\nabla\cdot M(\Z)$, and the time-irreversible drift $f_{\text{irr}} = -W(\Z)\nabla V(\Z) +\nabla\cdot W(\Z)$ are uniquely determined.
\end{theorem}

To prove this, we first need the following lemma.
\begin{lemma}\label{thm:dfode}
Let \( g: \mathbb{R}^D \to \mathbb{R}^D \) be continuously differentiable and satisfy \( \nabla \cdot g = 0 \). Then there exist \( D-1 \) Hamiltonian functions \( H_d : \mathbb{R}^D \to \mathbb{R} \), \( d = 1, \dots, D-1 \), such that
\[
g = \sum_{d=1}^{D-1} J_d \nabla H_d,
\]
where each \( J_d \) is an antisymmetric \( D \times D \) matrix whose only nonzero entries are \( +1 \) at \( (d, d+1) \) and \( -1 \) at \( (d+1, d) \). Each \( J_d \) couples consecutive dimensions in a Hamiltonian manner via \( H_d \).
\end{lemma}

This result, known as the essentially Hamiltonian decomposition, was proved by Feng and Shang~\cite{feng1995volume} for constructing volume-preserving integrators for divergence-free ODEs. For completeness, we provide its proof.
\begin{proof}
Let \( g = (g_1, \dots, g_D)^\top \) and \(z= (z_1, z_2, \dots, z_D)\in \R^D\). We then construct functions \( H_{d,d+1} \) satisfying the following $D$ constrains:
\[
g_1 = -\frac{\partial H_{1,2}}{\partial z_2}, \quad g_d = \frac{\partial H_{d-1,d}}{\partial z_{d-1}} - \frac{\partial H_{d,d+1}}{\partial z_{d+1}} \quad \text{for } d = 2, \dots, D-1, \quad g_D = \frac{\partial H_{D-1,D}}{\partial z_{D-1}}.
\]
First, define
\[
H_{1,2}(z) = -\int_0^{z_2} g_1(z_1, s, z_3, \dots, z_D)  ds,
\]
and for $d = 2, \dots, D-2$, define recursively:
\[
\quad H_{d,d+1}(z) = \int_0^{z_{d+1}} \left( \frac{\partial H_{d-1,d}}{\partial z_{d-1}} - g_d \right)(z_1, \dots,z_d, s, z_{d+2},\cdots, z_D) ds.
\]
It is ready to  verify that the first $D-2$ constrains are satisfied.

Next we construct $H_{D-1,D}$. A straightforward induction shows that for \( d \leq D-2 \),
\[
\frac{\partial^2 H_{d,d+1}}{\partial z_d \partial z_{d+1}} = -\left( \frac{\partial g_1}{\partial z_1} + \cdots + \frac{\partial g_d}{\partial z_d} \right).
\]
Using the identity \( \nabla \cdot g = 0 \), we conclude
\[
\frac{\partial}{\partial z_{D-1}} \left( \frac{\partial H_{D-2,D-1}}{\partial z_{D-2}} - g_{D-1} \right) = \frac{\partial g_D}{\partial z_D}.
\]
Therefore we can define
\[
H_{D-1,D}(z) = \int_0^{z_D} \left( \frac{\partial H_{D-2,D-1}}{\partial z_{D-2}} - g_{D-1} \right)(z_1, \dots, z_{D-1}, s) ds + \int_0^{z_{D-1}} g_D(z_1, \dots, z_{D-2}, t, 0)  dt,
\]
which satisfies the remaining constraints and completes the proof.
\end{proof}

Now we are ready to present the proof of Theorem~\ref{the:appro}.
\begin{proof}[Proof of Theorem~\ref{the:appro}]
We begin by defining the time-reversible part of the drift function $g$ as:
\[
g_{\text{rev}} = \frac{1}{\rho} \nabla \cdot (M \rho) = M\nabla \log \rho + \nabla \cdot M
\]
where $M$ is the diffusion matrix and $\rho$ is the stationary density. The time-irreversible part is then given by
\(
g_{\text{irr}} = g-g_{\text{rev}}
\).
Since $\rho$ satisfies the stationary Fokker-Planck equation~\cite{pavliotis2014stochastic}:
\[
    \nabla \cdot (g\rho - \nabla \cdot (M \rho))=0,
\]
according to the equality that
\[
g\rho - \nabla \cdot (M \rho) = g_{\text{irr}}\rho + g_{\text{rev}}\rho- \nabla \cdot (M \rho) = g_{\text{irr}}\rho,
\]
we obtain that
\[
\nabla \cdot (g_{\text{irr}}\rho)=0.
\]
Thus, $g_{\text{irr}}$ coincides with the current velocity and is divergence-free with respect to the stationary density $\rho$.

By Lemma~\ref{thm:dfode}, there exist functions $\tilde{H}_d$ such that:
\begin{equation*}
g_{\text{irr}} \rho = \sum_{d=1}^{D-1} J_d \nabla \tilde{H}_d.
\end{equation*}
Define new Hamiltonian functions $H_d = \rho^{-1} \tilde{H}_d$. Applying the product rule yields:
\begin{equation*}
g_{\text{irr}} \rho = \sum_{d=1}^{D-1} J_d\rho \nabla H_d  - \sum_{d=1}^{D-1}H_d J_d \nabla \rho.
\end{equation*}
Now, introduce the potential $V = -\log \rho$ and the antisymmetric matrix $W = \sum_{d=1}^{D-1}H_d J_d$, which indicates
\[
\nabla \rho = - \rho \nabla V, \quad \nabla\cdot W = \sum_{d=1}^{D-1} J_d \nabla H_d.
\]
Clearly, $W(\Z)$ is a banded anti-symmetric matrix whose non-zero entries are confined to the first sub-diagonal and its corresponding counterparts:
\begin{equation*}
\begin{aligned}
W(\Z) =& \begin{pmatrix}
    0& H_1(\Z)&0 &\cdots&0&0\\
    -H_1(\Z)&0&H_2(\Z)&\cdots&0&0\\
    0&-H_2(\Z)&0&\cdots&0&0\\
    \vdots&\vdots&\vdots& \ddots&\vdots&\vdots\\
    0&0&0&\cdots&0&H_{D-1}(\Z)\\
    0&0&0&\cdots&-H_{D-1}(\Z)&0\\
\end{pmatrix}.
\end{aligned}
\end{equation*}
Substituting into the expression above, we obtain:
\[
g_{\text{irr}} = -W\nabla V + \nabla \cdot W, \quad g_{\text{rev}} = -M\nabla V + \nabla \cdot M,
\]
and thus
\begin{equation*}
    g = g_{\text{irr}}+ g_{\text{rev}}=-[M + W ]\nabla V  + \nabla\cdot M  + \nabla\cdot W .
\end{equation*}

Next we prove the uniqueness. We have shown in subsection~\ref{app: stationary density} that the density
\[
\rho(\Z) = \frac{1}{\mathcal{Z}} {e^{-V(\Z)}}
\]
is a stationary density of the system. By the uniqueness assumption of the stationary density, the potential \( V \) in the~\eqref{eq: app_OnsagerHHD} is uniquely determined up to an additive constant and satisfies
\[
V = -\log \rho.
\]
Hence, the potential force \( \nabla V \) is uniquely given by
\[
\nabla V = -\nabla \log \rho .
\]

Since \( M \) is the fixed diffusion matrix, the time-reversible part of the drift is also uniquely determined as
\[
f_{\text{rev}} = -M \nabla V + \nabla \cdot M = M \nabla \log \rho + \nabla \cdot M.
\]
Consequently, the irreversible part of the drift,
\[
f_{\text{irr}} = g - f_{\text{rev}},
\]
is uniquely determined.
\end{proof}

\section{Algorithm\label{sec:methods}}
In this subsection, we detail the architecture of the {\ourmethodname} modules and outline the practical implementation for learning.

\subsection{Neural Network Parameterization}
To ensure that the diffusion matrix $M(\Z) = \sigma(\Z) \sigma(\Z)^{\top} / 2$ is positive definite, we parameterize the diffusion term $\sigma(\Z)$ using a structured representation. For applications of SGLD under state-dependent conditions, we define:
$$\sigma(\Z) = \sigma_{1}(\Z) + \mathrm{diag}\left( \frac{ \sqrt{ \sigma_{2}(\Z)^2 + 1 } + \sigma_{2}(\Z) }{2} \right),$$
where $\sigma_{1}(\Z) \in \mathbb{R}^{D \times D}$ is a strictly lower-triangular matrix whose entries are formed by reshaping the output of a neural network, and $\sigma_{2}(\Z) \in \mathbb{R}^D$ is generated by a separate network. This construction, analogous to a Cholesky decomposition, guarantees positive diagonal entries, thus ensuring that $M(\Z)$ is positive definite.
Since any symmetric positive definite matrix admits a unique Cholesky decomposition, learning the components of this decomposition with universal approximators allows the model to represent any arbitrarily complex, state-dependent diffusion matrix.
In the application to \nameref{sec:sgld}, the networks for both $\sigma_{1} $ and $\sigma_{2}$ consist of two hidden layers with 32 neurons each and use the tanh activation function. In the application to \nameref{sec:polymer}, we empirically found that a state-independent diagonal diffusion matrix performed best. This was implemented using a linear network with zero weights and a trainable diagonal bias.

The antisymmetric matrix $W$ is parameterized by a neural network $H : \mathbb{R}^{D} \rightarrow \mathbb{R}^{D-1}$ through the following linear combination:
$$W(\Z) = \sum_{d=1}^{D-1} H_d(\Z) J_d,$$
where $H_d$ denotes the $d$-th component of the output of the network $H$, and each $J_d$ is a fixed $D \times D$ antisymmetric matrix with nonzero entries of $+1$ and $-1$ at positions $(d, d+1)$ and $(d+1, d)$, respectively.
This construction yields a banded antisymmetric matrix $W(\Z)$, whose non-zero entries are confined to the first super-diagonal and sub-diagonal:
\begin{equation*}
W(\Z) \!=\!\!\begin{pmatrix}
    0& H_1(\Z)& \cdots&0&0\\
    -H_1(\Z)&0& \cdots&0&0\\
    0&-H_2(\Z)& \cdots&0&0\\
    \vdots&\vdots & \ddots&\vdots&\vdots\\
    0&0& \cdots&0&H_{D-1}(\Z)\\
    0&0& \cdots&-H_{D-1}(\Z)&0\\
\end{pmatrix}.
\end{equation*}
In both the polymer and SGLD experiments, the function $H$ is implemented as a neural network with two hidden layers of $128$ neurons each, using the tanh activation function.

To ensure integrability of \( e^{-V} \), the potential function \( V \) is designed as:
\[
V(\Z) = \frac{\beta_1}{2} \sum_{i=1}^{m} \left( U_i(\Z) + \beta_2 \sum_{j=1}^{D} \gamma_{ij} \Z_j \right)^2 + \beta_3 \| \Z \|^2,
\]
where \( U : \mathbb{R}^D \rightarrow \mathbb{R}^m \) is a neural network, \( \gamma_{ij} \) are trainable parameters, and \( \beta_1, \beta_2, \beta_3 > 0 \) are fixed scaling factors. This architecture ensures that the potential function satisfies both the integrability condition and the appropriate growth requirements specified in Methods~\ref{app: the}.
The expressivity is achieved through the embedded neural network
\(U\) within this constrained functional form.
In our implementation, we set \( m = 32 \) and use a network \( U \) with one hidden layer of 128 neurons. We employ the shifted ReQU activation in polymer dynamics and tanh in the SGLD context.

All trainable parameters across these neural networks are collectively denoted by $\theta$. Consequently, we write the parameterized functions as $ \sigma_{\theta} $, $ W_{\theta} $, and $V_{\theta}$ to explicitly highlight their dependence on these learnable parameters.

\subsection{Learning Algorithm}
Here, we formulate the loss function used to infer the stochastic dynamics from discretely observed trajectory data

The dataset consists of $N$ independent trajectories, each sampled at discrete time points \( t_0, t_1, \dots, t_M \), denoted as:
\[
\left\{ \Z^{(n)}_0, \Z^{(n)}_1, \dots, \Z^{(n)}_M \right\}_{n=1}^N.
\]
Each trajectory is assumed to be a realization of an underlying SDE starting from the initial condition $\Z^{(n)}_0$.
Our objective is to learn the functions \( V \), \( W \), and \( \sigma \) that govern the dynamics of the {\ourmethodname} within this deep learning framework.

As described in the previous subsection, we model \( V_\theta \), \( W_\theta \), and \( \sigma_\theta \) with neural networks, resulting in the following parametric SDE:
\[
\mathrm{d} \Z(t) = f_\theta(\Z_t)  \mathrm{d}t + \sigma_\theta(\Z_t)  \mathrm{d}\omega(t),
\]
where
\[
f_\theta = -\big[M_\theta + W_\theta \big] \nabla V_\theta + \nabla \cdot M_\theta + \nabla \cdot W_\theta,\quad
M_\theta = \frac{1}{2} \sigma_\theta \sigma^\top_\theta.
\]
We will use the Maximum Likelihood Estimation (MLE) to learn the parameter $\theta$~\cite[Sec. 5.3]{pavliotis2014stochastic}. Let $p(\Delta t, \Z \mid 0, \Z_0; \theta)$ be the transition density of the process (i.e., the conditional probability density of the system being in state $\Z$ at time $\Delta t$, given it started at $\Z_0$ at time $0$). The MLE objective is:
\[
\theta^* = \arg\max_\theta  \sum_{n=1}^N \sum_{m=1}^M \log p(\Delta t_m, \Z^{(n)}_m \mid 0, \Z^{(n)}_{m-1}; \theta),
\]
where $\Delta t_m = t_m - t_{m-1}$. Since the exact transition density is typically intractable for general nonlinear diffusion processes, we approximate it using a numerical approximation $p_h(\Delta t, \Z \mid 0, \Z_0; \theta)$ based on the method developed in~\cite{zhu2024dyngma}. The practical optimization problem thus becomes:
\[
\theta^* = \arg\max_\theta  \sum_{n=1}^N \sum_{m=1}^M \log p_h(\Delta t_m, \Z^{(n)}_m \mid 0, \Z^{(n)}_{m-1}; \theta).
\]
Specifically, the one-step scheme from~\cite{zhu2024dyngma} leads to a Gaussian approximation for the transition density:
\[
\begin{aligned}
&p_h(\Delta t, \Z \mid 0, \Z_0; \theta) \\
=& \mathcal{N}\left(\Z \mid \Z_0 + \Delta t  f_\theta(\Z_0),  \Delta t  \sigma_\theta(\Z_0) \sigma^\top_\theta(\Z_0) \right)\\
=&\frac{(2\pi)^{-1/2}}{\sqrt{\det\left[\Delta t\sigma_\theta(\Z_0) \sigma^\top_\theta(\Z_0)\right]}} \exp\Big(\!\!-\!\frac{1}{2} \left[\Z - \left(\Z_0 + \Delta t f_\theta(\Z_0)\right)\right]^\top \\
&\left[\Delta t \sigma_\theta(\Z_0) \sigma^\top_\theta(\Z_0)\right]^{-1} \left[\Z - \left(\Z_0 + \Delta t f_\theta(\Z_0)\right)\right]\Big).
\end{aligned}
\]
This approximation, corresponding to the Euler-Maruyama scheme, has been widely used in related works~\cite{chen2024constructing, dietrich2023learning}. We employ the Adam optimizer~\cite{kingma2014adam} for training, with detailed hyperparameter settings provided in Methods~\ref{app:Supplementary Information for Numerical Results}.

\section{Supplementary Information for Numerical Results}\label{app:Supplementary Information for Numerical Results}

\subsection{Properties of Global EPR}\label{app: properties of Global EPR}
In this subsection, we briefly review the necessary properties of global EPR.

\paragraph{Coarse-graining by Projection}\label{app: property of EPR}
Here we show that the coarse-graining by projection cannot increase the global EPR by the conditional KL decomposition.
\begin{proposition}[Conditional KL Decomposition]\label{pro:con kl}
Let $(\mathcal{X},\mathcal{A})$ and $(\mathcal{Y},\mathcal{B})$ be standard Borel spaces.
Let $P_{XY}$ and $Q_{XY}$ be probability measures on $(\mathcal{X}\times\mathcal{Y},\mathcal{A}\otimes\mathcal{B})$
such that $P_{XY}\ll Q_{XY}$ and $\mbox{KL}[P_{XY}, Q_{XY}]<\infty$.
Assume that regular conditional probabilities \(P_{X\mid Y}(\cdot\mid y)\) and \(Q_{X\mid Y}(\cdot\mid y)\) exist.
Then
\begin{equation*}
\mbox{KL}[P_{XY}, Q_{XY}]
= \mbox{KL}[P_Y, Q_Y] \;+\; \mathbb{E}_{P_Y}\big[ \mbox{KL}[P_{X\mid Y}(\cdot\mid Y), Q_{X\mid Y}(\cdot\mid Y)] \big].
\end{equation*}
\end{proposition}

This decomposition is a well-established result in information theory~\cite{cover1999elements}.
For completeness, we provide a proof in the case where densities exist.

\begin{proof}
Suppose $P_{XY}$ and $Q_{XY}$ admit densities $p$ and $q$ with respect to a common dominating measure.
Factor the joint densities using regular conditional probabilities:
\[
p(x,y) = p(x\mid y)p(y),\qquad q(x,y) = q(x\mid y)q(y).
\]
Thus,
\[
\log\frac{p(x,y)}{q(x,y)} = \log\frac{p(y)}{q(y)} + \log\frac{p(x\mid y)}{q(x\mid y)}.
\]
Integrating with respect to $p(x,y)\,\d x\,\d y$ gives
\[
\begin{aligned}
\mbox{KL}[P_{XY}, Q_{XY}]
&= \int p(x,y)\log\frac{p(x,y)}{q(x,y)}\,\d x\,\d y \\
&= \int p(x,y)\log\frac{p(y)}{q(y)}\,\d x\,\d y
  + \int p(x,y)\log\frac{p(x\mid y)}{q(x\mid y)}\,\d x\,\d y \\
&= \int p(y)\log\frac{p(y)}{q(y)}\,\d y
  + \int p(y)\Bigg[\int p(x\mid y)\log\frac{p(x\mid y)}{q(x\mid y)}\,\d x\Bigg]\d y \\
&= \mbox{KL}[P_Y, Q_Y] + \mathbb{E}_{P_Y}\big[ \mbox{KL}[P_{X\mid Y}(\cdot\mid Y), Q_{X\mid Y}(\cdot\mid Y)] \big].
\end{aligned}
\]
\end{proof}
The conditional KL decomposition formalizes the property that the entropy production computed from the projected dynamics cannot exceed that of the full system.
Let $\Z_{t}=(\Z^1_t,\Z^2_t)$ be the full trajectory and $\phi(\Z_{t}) = \Z^1_{t}$ be the projection to the $\Z^1$-component.
Applying Proposition~\ref{pro:con kl} we have:
\[
\mbox{KL} [\mathbf{P}_{[\tau, \tau+\varepsilon]},\ \bar{\mathbf{P}}_{[\tau, \tau+\varepsilon]} ] = \mbox{KL}[\mathbf{P}^1_{[\tau, \tau+\varepsilon]},\ \bar{\mathbf{P}}^1_{[\tau, \tau+\varepsilon]} ] + \mathbb{E}_{\mathbf{P}^1_{[\tau, \tau+\varepsilon]}}\left[ \mbox{KL}[\mathbf{P}_{[\tau, \tau+\varepsilon]}(\cdot \mid \Z^1_t),\ \bar{\mathbf{P}}_{[\tau, \tau+\varepsilon]}(\cdot \mid \Z^1_t) ] \right],
\]
where $\mathbf{P}^1_{[\tau, \tau+\varepsilon]} = \mathbf{P}_{[\tau, \tau+\varepsilon]} \circ \phi^{-1}$ and $\bar{\mathbf{P}}^1_{[\tau, \tau+\varepsilon]} = \bar{\mathbf{P}}_{[\tau, \tau+\varepsilon]} \circ \phi^{-1}$.
This decomposition, together with the nonnegativity of KL divergence, yields
\[
\mbox{KL} [\mathbf{P}_{[\tau, \tau+\varepsilon]},\ \bar{\mathbf{P}}_{[\tau, \tau+\varepsilon]} ] \geq  \mbox{KL}[\mathbf{P}^1_{[\tau, \tau+\varepsilon]},\ \bar{\mathbf{P}}^1_{[\tau, \tau+\varepsilon]} ].
\]

\paragraph{Direct Numerical Computation}\label{app: epr cal}
In this subsection, we present a numerical method for computing the global EPR.

The global EPR can be quantified by measuring time-irreversibility according to the path-based definition:
\begin{equation*}
\dot{S}_{\text{tot}} = \lim_{\varepsilon \downarrow 0} \frac{1}{\varepsilon} \mbox{KL}\left[\mathbf{P}_{[\tau, \tau+\varepsilon]},\ \bar{\mathbf{P}}_{[\tau, \tau+\varepsilon]}  \right],
\end{equation*}
where $\mathbf{P}$ and $\bar{\mathbf{P}}$ denote path space measures for the forward process $\Z_t$ and its time reversal $\bar{\Z}_t$, respectively.

We numerically compute $\dot{S}$ from discrete trajectories with time step $\Delta t$ by approximately implementing its definition, following the procedure below.
First, the continuous state space of variables $\Z$ is discretized into a uniform grid with $N_{\text{bins}}$ bins per dimension. For each trajectory, we generate its time-reversed counterpart by inverting the temporal sequence of states. Transition probabilities $\mathbf{P}(\Z\rightarrow\Z')$ for the forward process and $\bar{\mathbf{P}}(\Z\rightarrow\Z')$ for the time-reversed process are then estimated from frequency counts of bin-to-bin transitions. Finally, the global EPR is calculated as:
\begin{equation*}
\dot{S}_{\text{tot}} \approx \frac{1}{\Delta t} \left\langle \ln \frac{\mathbf{P}(\Z\rightarrow\Z')}{\bar{\mathbf{P}}(\Z\rightarrow\Z')} \right\rangle
\end{equation*}
where $\langle \cdot \rangle$ denotes averaging over all observed transitions in the forward process. For transitions unobserved in the time-reversed process ($\bar{\mathbf{P}} = 0$), we set $\bar{\mathbf{P}}(\Z\rightarrow\Z') = \min_{\bar{\mathbf{P}} > 0} \bar{\mathbf{P}} / 10$, where the minimum is taken over all observed transition probabilities in the time-reversed process.
This approach is only applicable to systems with very few degrees of freedom. And its computational cost is substantial, as the procedure involves evaluating the KL divergence between path measures over trajectories sampled from the invariant distribution.

\subsection{Barrier Heights}\label{app:polymer bh}

Since a key objective is to calculate the barrier height from the learned energy function, we provide a brief introduction to the concept.

The barrier height for the transition from a stable state \(A\) to another stable state \(B\) is defined as the minimal additional value of an effective energy function $E$ that the system must acquire to escape from \(A\) along the optimal transition path connecting \(A\) and \(B\)~\cite{hanggi1990reaction, van1992stochastic}. Formally, let \(E^\ddagger\) denote the maximum energy along the minimum energy path between \(A\) and \(B\); the barrier height is then given by
\[
\Delta E_{A \to B} = E^\ddagger - E(A).
\]
Here \(E^\ddagger\) corresponds to the energy of the first-order saddle point on this path.

To numerically compute the barrier height between two stable states, we proceed as follows. First, we select two points located near the respective stable states. Each given point is first relaxed by gradient descent to determine its corresponding local minimum, identified as the stable states of interest. Then we seek the first-order saddle point on the minimum-energy path connecting the two minima. This is achieved using the Gentlest Ascent Dynamics (GAD) method~\cite{weinan2011gentlest}, which modifies the gradient flow so that the dynamics climb along the unstable direction while descending along all others. Specifically, for a system with potential energy $E(x)$, GAD evolves both the position $x$ and a direction vector $v$ according to
\[
\dot{x} = -\nabla E(x) + 2\frac{ v^\top \nabla E(x) }{v^\top v}v,\qquad
\dot{v} = -\nabla^2 E(x)\,v + \frac{v^\top \nabla^2 E(x)\,v}{v^\top v}\,v.
\]
The first equation reverses the component of the gradient along $v$, causing the system to ascend along the least-stable direction while descending in all others, so that the trajectory converges to an index-1 saddle point. There may exist multiple saddle points along different transition routes. We run GAD from several initial points between the minima to generate multiple saddle candidates and select the one with the lowest energy. The barrier from\(A\) to
\(B\) is then defined as the energy difference between this lowest-energy saddle and \(A\). From the energy landscape considered in this work, the system contains only two stable states; therefore, intermediate minima case is not considered here.

Typically, the effective energy function is characterized by the quasi-potential.
In the small noise limit, it can be asymptotically approximated by the negative logarithm
of the stationary density,
scaled by the noise intensity~\cite{biswas2009small, freidlin2012random}.
In this paper, {\ourmethodname}s yield the negative logarithm of the stationary density
under varying external forces at fixed temperature,
ensuring that the resulting trends in barrier heights are consistent with
the theoretical definition.
Although the energy function returned by OnsagerNet lacks a strict physical definition,
we estimate its barrier height using the same procedure for comparative analysis.

\subsection{Linear Case}\label{app:linear}
We describe the data generation procedure.
The dataset is created by simulating a target linear SDE of the form:
\[
\dot{\Z}_t = -\bigl(M + W\bigr) S\, \Z_t + \sqrt{2M}\, \dot{\mathbf{B}}_t,
\]
where \(M\) and \(S\) are symmetric positive-definite matrices, and \(W\) is antisymmetric with \(W = \lambda W_0\). This formulation is designed to capture the combined effects of symmetric dissipation, antisymmetric coupling, and stochastic forcing.

The coefficient matrices \(M\), \(W\), and \(S\) are obtained by first sampling Gaussian random matrices and normalizing them. Subsequently, symmetry constraints are imposed: \(M\) and \(S\) are symmetrized and scaled to ensure positive definiteness, whereas \(W\) is rendered strictly antisymmetric. As a result, the generated system satisfies the intended structural properties by construction.

Once the matrices are fixed, the SDE is integrated over the time interval \([0, 1]\) using a constant time step of \(0.01\). Here, initial conditions are drawn from a Gaussian distribution with zero mean and standard deviation \(2\), and each trajectory is driven by an independent Brownian motion.
In total, \(10^4\) independent trajectories are simulated, of which \(9000\) are allocated to the training set and \(1000\) are reserved for the test set.

The potential module ($V$) consists of two fully connected layers with 64 and 32 hidden units, respectively, activated by the ReCU function.
The anti-symmetry module ($W$) includes two hidden layers with 32 and 16 units, also using the ReCU activation.
The entire model is trained using the Adam optimizer for 5000 epochs with a batch size of 50,000.

\subsection{Polymer Stretching Dynamics}\label{app:polymer}

\paragraph{Simulation Data Preparation}
This subsection describes the workflow for generating the simulation dataset, including modelling and simulation of polymer dynamics at the microscopic scale and computing representative coordinates at the macroscopic scale.

\textbf{Microscopic polymer dynamics simulation.}
We begin by describing the model for microscopic polymer dynamics. The essential features of the governing equations are summarized below; additional details can be found in~\cite{chen2024constructing}.

We employ a Brownian dynamics framework to simulate linear bead-rod polymer chains subjected to planar elongational flow. Each chain consists of $N = 300$ beads, corresponding to $D = 3N$ positional degrees of freedom, each with diameter $r$. Adjacent beads are linked by $N-1$ rigid rods of uniform length $b$, with $b = r = 10\,\mathrm{nm}$. The stochastic differential equation governing bead motion accounts for excluded-volume interactions, constraint forces, Brownian fluctuations, and hydrodynamic drag.
The excluded-volume interaction, which models short-range repulsion between beads, is described by the potential:
\[
E_{\mathrm{ev}} = -\sum_{i<j} \mu\, r_{ij} \quad \text{for} \quad r_{ij} < r,
\]
where $r_{ij}$ denotes the distance between beads $i$ and $j$, and $\mu = 35\ \mathrm{pN}$ is chosen to effectively suppress chain crossings, as validated in~\cite{vologodskii2006brownian}.
Constraint forces enforce fixed bond lengths and are expressed as:
\[
\mathbf{F}^c_i = T_i \mathbf{b}_i - T_{i-1} \mathbf{b}_{i-1},
\]
where $\mathbf{b}_i$ is the unit vector along bond $i$, and $T_i$ is the tension in that bond.
Brownian forces $\mathbf{F}^{\mathrm{br}}_i(t)$ are modelled as Gaussian white noise satisfying the fluctuation-dissipation theorem:
\[
\langle \mathbf{F}^{\mathrm{br}}_i(t) \rangle = \mathbf{0}, \quad
\langle \mathbf{F}^{\mathrm{br}}_i(t) \, \mathbf{F}^{\mathrm{br}}_j(t') \rangle = \frac{2 k_B T \zeta\, \delta_{ij} \, \mathbf{I} }{\Delta t} \, \delta(t - t'),
\]
where $\delta_{ij}$ is the Kronecker delta, $\mathbf{I}$ is the identity matrix, and $\Delta t$ is the simulation time step.
Hydrodynamic interactions are neglected. Each bead experiences a Stokes drag force:
\[
\mathbf{F}^d_i = -\zeta \left( \mathbf{u}(\mathbf{r}_i) - \dot{\mathbf{r}}_i \right),
\]
with drag coefficient $\zeta = 3\pi \eta r$ (where $\eta$ is the solvent viscosity), and $\mathbf{u}(\mathbf{r}_i)$ is the unperturbed solvent velocity at bead position $\mathbf{r}_i$.
Under the assumption of negligible inertia, force balance yields the overdamped Langevin equation for bead motion:
\[
\dot{\mathbf{r}}_i = \mathbf{u}(\mathbf{r}_i) + \frac{1}{\zeta} \left( \mathbf{F}^{\mathrm{ev}}_i + \mathbf{F}^c_i + \mathbf{F}^{\mathrm{br}}_i \right).
\]

Numerical integration is performed using a predictor-corrector algorithm~\cite{liu1989flexible}. The rigid-rod constraints lead to a non-linear system of equations for the tensions $T_i$, which is solved at each time step via Newton’s method~\cite{somasi2002brownian}.

Each trajectory begins with an equilibration phase lasting $10^4\tau_d$, where $\tau_d = b^2 \zeta / (k_B T)$ is the characteristic rod diffusion time. This allows the chain to sample equilibrium conformations. At $t = 0$, the equilibrated chain is subjected to a planar elongational flow:
\[
\mathbf{u}(\mathbf{r}_i) = \dot{\epsilon} \, ( \hat{\mathbf{x}} - \hat{\mathbf{y}})\cdot \mathbf{r}_i,
\]
where $\dot{\epsilon}$ is the strain rate, and $\hat{\mathbf{x}}, \hat{\mathbf{y}}$ are unit vectors along the $x$- and $y$-axes, respectively. Production simulations are run until $t = 10^4\tau_d$ with a time step $\Delta t = 5 \times 10^{-4}\tau_d$. The three-dimensional coordinates of all $N$ beads are recorded every $10\tau_d$.

For each value of the control parameter $F$ (number of independent stretching trajectories), we simulate multiple realizations. For testing, we generate 500 trajectories from each of three distinct initial conformations, selected to represent qualitatively different dynamical behaviors. The full 3D coordinates $\mathbf{r}_i(t) = (x, y, z)$ of all beads are stored at each sampling time.
~\\

\textbf{Macroscopic coordinates computation.}
Based on the recorded microscopic coordinates, the macroscopic chain extension can be computed as:
\[
Z_1 = |\mathbf{r}_1 - \mathbf{r}_N|.
\]
We outline the procedure for obtaining the remaining closure coordinates within a dimensionality reduction framework.
Specifically, we apply the PCA-ResNet approach~\cite{chen2024constructing} to extract these coordinates:
\[
Z(t) = P_d X(t) + \mathrm{NN}_e(X(t)),
\]
where $P_d$ contains the leading $d-1$ principal components, and $\mathrm{NN}_e$ is a non-linear encoder that captures correlations beyond the linear subspace. Reconstruction is performed via:
\[
\hat{X}(t) = P_d^{+} Z(t) + \mathrm{NN}_a(Z(t)),
\]
with $P_d^{+} \in \mathbb{R}^{n \times d}$ denoting the pseudoinverse of $P_d$, and $\mathrm{NN}_a$ a non-linear decoder. The reconstruction loss is defined as:
\[
\mathcal{L}_{\mathrm{rec}} = \frac{1}{T} \sum_{t=1}^T \| X(t) - \hat{X}(t) \|_2^2.
\]
The PCA baseline error $E_{\mathrm{PCA}}$ is computed similarly using only $P_d$. The model is trained by minimizing the loss:
\[
\mathcal{L} = \mathcal{L}_{\mathrm{rec}} + \lambda \, \mathrm{ReLU}\left( \log \mathcal{L}_{\mathrm{rec}} - \log E_{\mathrm{PCA}} \right),
\]
which ensures that the non-linear model does not underperform the linear benchmark. This approach yields a low-dimensional set of physically interpretable variables that faithfully represent the polymer stretching dynamics.

Although the training data are derived from explicit bead-rod simulations,
{\ourmethodname} operates only on reduced coordinates.
This property decouples the learning framework from system-specific details,
allowing its application to a broad range of non-equilibrium systems.

\paragraph{Experimental Data Preparation}\label{app: exp polymer data}
\begin{figure}[t]
    \centering
    \includegraphics[width=1\linewidth]{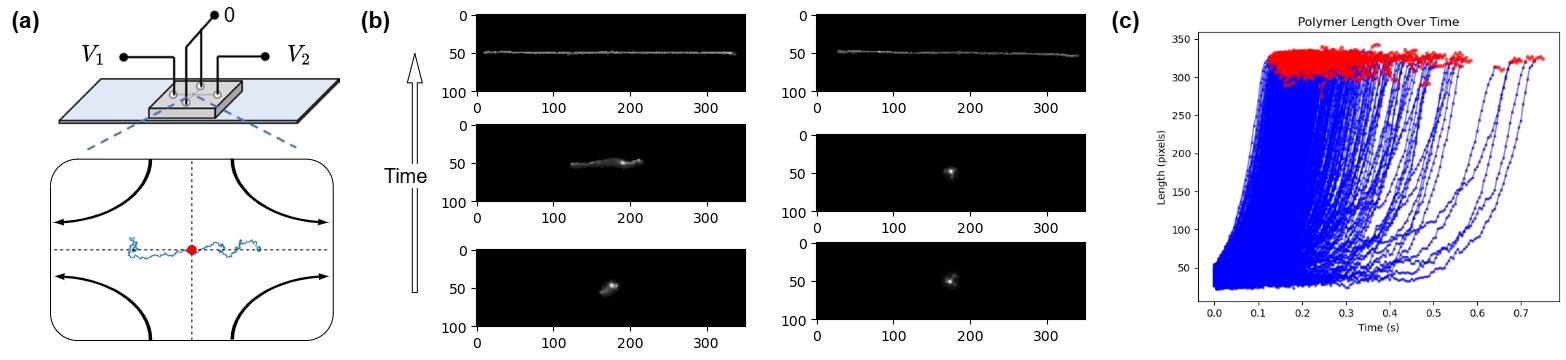}
    \caption{ \textbf{(a)} The schematic and photograph show our experimental system. A cross-slot microfluidic device is equipped with platinum electrodes in each of its four reservoirs. We create an electric field by applying positive voltages $V_1$ and $V_2$ to the east (E) and west (W) reservoirs, while the north (N) and south (S) reservoirs serve as the ground. This voltage configuration guides negatively charged DNA from the grounded reservoirs, pulling it into the channel's center. The process culminates when a single molecule is captured and held in place, a point marked by the blue dot in the schematic. \textbf{(b)} Snapshots showing two distinct DNA stretching processes. \textbf{(c)} Time evolution of DNA length, The red points correspond to the selected DNA configurations used to evaluate the global EPR.}
    \label{fig:exp_polymer_set}
\end{figure}

An automated single-molecule stretching trap is implemented to obtain experimental image data, as originally developed in~\cite{soh2023automated}. In the following, we introduce the key aspects of the setup. An illustration of the experimental data preparation is provided in Fig.~\ref{fig:exp_polymer_set}. Comprehensive descriptions of the experimental protocol and material preparation are available in~\cite{soh2023automated}.

The experiments are conducted using a microfluidic cross-slot channel device. The device has a central chamber from which four 40~$\mu$m-wide channels extend in the cardinal directions. Platinum wire electrodes are positioned at the ends of these channels and are connected to a computer-controlled voltage source.

A quadrupolar electric field is established by grounding the North and South electrodes and applying a positive
bias of $+60V$ ($V_0$) to the East and West electrodes. This configuration creates a potential well in the North-South direction and a potential hill in the East-West direction, with an unstable saddle point at the central chamber. Negatively charged T4 phage double-stranded DNA molecules (165.6 kbp), fluorescently labeled with YOYO-1 for visualization, serve as the polymer samples. As the molecules drift electrokinetically, they are drawn towards this saddle point.

To enable long-term observation and controlled stretching, an active feedback system is implemented to stabilize the saddle point.
Images are acquired at 50 ms intervals using an inverted fluorescence microscope equipped with a high-resolution sCMOS camera. A custom LabVIEW program processes the image stream in real time to extract the molecular centroid and projected contour length.
A proportional feedback loop with a gain of \(G = 2.2~\mathrm{V}/\mu\mathrm{m}\) adjusted the East-West voltage to maintain the DNA centroid within \(1~\mu\mathrm{m}\) of the saddle point during stretching, thereby allowing stable trapping and extension of the molecule in a planar elongational field over extended periods.

The data acquisition process is fully automated. A single molecule is actively trapped. The East-West voltage ($V_0$) is then temporarily set to zero for 10 seconds to allow the molecule to relax into an unstretched, equilibrium state. After this relaxation period, recording begins with images being streamed to the computer's solid-state drive. $V_0$ is then reset to $+60V$, which re-establishes the East-West potential hill and stretches the DNA molecule. By monitoring the projected length history, the platform recognizes when the molecule is fully stretched. In response, the recording is stopped, and the molecule is released to escape naturally towards the East or West.
This protocol enabled the acquisition of a broad range of stretching trajectories; here, 605 such trajectories are collected and employed to evaluate the global EPR.
To ensure that the system had reached a steady state, only configurations in which the DNA chain is fully extended are selected.
From the image-recorded configurations, we extract the time series of the horizontal coordinates of the left and right endpoints.

\begin{figure}[!ht]
    \centering
    \includegraphics[width=\linewidth]{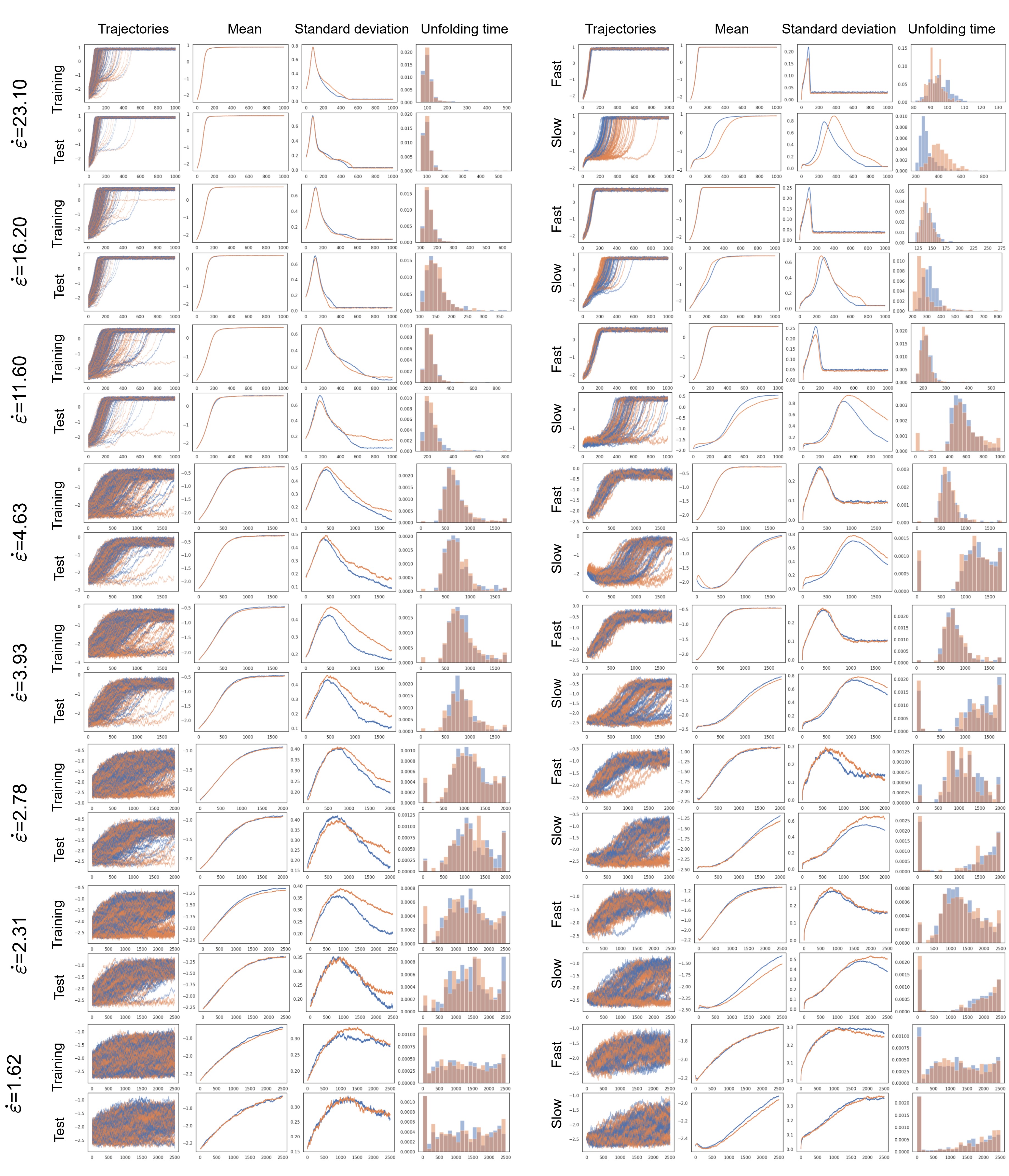}
    \caption{Comparison of predicted and true stretching dynamics, together with their statistical characteristics, for training trajectories, standard test trajectories, and supplementary fast and slow test trajectories.}
    \label{fig:ploymer_sta}
\end{figure}
\begin{figure}[!ht]
    \centering
    \includegraphics[width=1\linewidth]{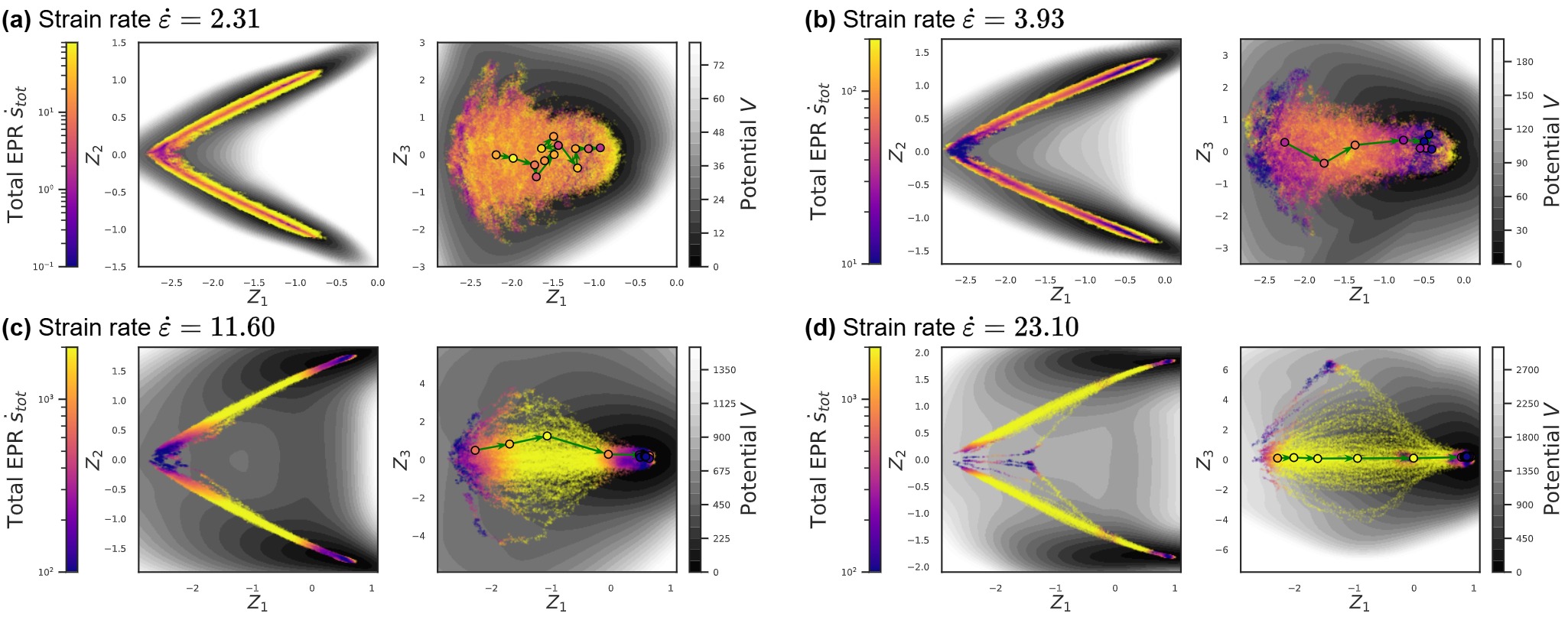}
    \caption{The total EPR  $\dot{s}_{\text{tot}}$ on test trajectories projected onto the $Z_1\text{-}Z_2$ and $Z_1\text{-}Z_3$ planes. The background is the corresponding potential energy surface.
    One example trajectory from the test dataset is overlaid on the $Z_1\text{-}Z_3$ panels.}
    \label{fig:local_epr_si}
\end{figure}

\paragraph{Additional Numerical Results}\label{app: Additional Numerical Results polymer}
Training is performed for 2000 epochs with a batch size of 2. Each batch contains two complete trajectories, which allows efficient optimization while preserving the full temporal structure of the training sequences.

To assess the accuracy of the learned {\ourmethodname}, we evaluate its performance on both the training and test trajectories.  Moreover, validation is performed using two distinct initial chain conformations absent from the training set. These representative configurations were deliberately initialized with near-identical projected extensions yet diverged significantly in their transient stretching dynamics, with one stretching much faster than the other.
As shown in Fig.~\ref{fig:ploymer_sta}, on both the training and test datasets across all cases, {\ourmethodname} precisely captures these contrasting behaviors and reproduces the true ensemble evolution of chain extension under varying external force.
{\ourmethodname} also achieves high-fidelity reconstruction of the first-passage time distributions required to attain a specified extension threshold.

For completeness, we provide the results for all examined strain rates $\dot{\varepsilon}$ in Fig.~\ref{fig:local_epr_si}. The behavior of the total EPR remains qualitatively consistent with that reported in the main text, exhibiting a transition from nearly reversible dynamics at small $\dot{\varepsilon}$ to strongly irreversible unfolding at large $\dot{\varepsilon}$.

\subsection{Stochastic Gradient Langevin Dynamics}\label{app:sgld}

\paragraph{Least Squares Regression}
In this paper, we leverage {\ourmethodname} to learn the dynamics of the SGLD discrete-time process.
As an example, consider the least squares problem:
\[
\min_{\Z} \; L(\Z) = \frac{1}{2}\|\mathbf{A}\Z - \mathbf{v}\|^2,
\]
where \(\mathbf{A} \in \mathbb{R}^{n \times D}\) is the data matrix, and \(\mathbf{v} \in \mathbb{R}^{D}\) is the observation vector. By expressing the objective function in terms of the individual rows of \(\mathbf{A}\), we can write:
Decomposing $L(\Z)$ into per-sample losses yields
\[
L(\Z) = \frac{1}{2} \sum_{i=1}^n \left( \mathbf{a}_i^T \Z - b_i \right)^2,
\]
with \(\mathbf{a}_i^T\) denoting the \(i\)-th row of \(\mathbf{A}\) and \(v_i\) denoting the \(i\)-th element of \(\mathbf{v}\). Applying SGLD with step size $\eta$ then generates the iterates:
Using SGLD with a randomly selected set of indices \( \{\gamma_1, \gamma_2, \ldots, \gamma_b\} \), the update rule becomes:
\[
\Z^{(t+1)} = \Z^{(t)} - \eta \frac{n}{b} \sum_{i=1}^b \left( \mathbf{a}_{\gamma_i}^T \Z^{(t)} - v_{\gamma_i} \right) \mathbf{a}_{\gamma_i} + \sqrt{2\eta} \ \boldsymbol{\xi}^{(t)},
\]
where \( \{\gamma_1, \gamma_2, \ldots, \gamma_b\} \) is a randomly selected set of indices. $\boldsymbol{\xi}^{(t)} \sim \mathcal N(\mathbf 0,\mathbf I)$ is standard Gaussian noise.
It is remarked that when $n=1$ and noise is omitted,
this reduces to the randomized Kaczmarz method~\cite{karczmarz1937angenaherte,strohmer2009randomized}.
And we use it as a test case for evaluating linear SGLD dynamics.

We employ the data matrix \textit{relat4} from the University of Florida Sparse Matrix Collection\footnote{https://sparse.tamu.edu/}.
The matrix originates from a combinatorial optimization problem, has a size of $66\times12$, and yields a 12-dimensional SGLD.
In our implementations, Gaussian random noise is incorporated into these matrices to guarantee full rank.
The solution vector \( \Z^* \) is sampled from a uniform distribution, and the corresponding right-hand side $b$ is obtained as \( A\Z^* \) with an added Gaussian perturbation.

To generate trajectory data, each simulation is initialized from a Gaussian distribution with mean~5 and standard deviation~3, and evolved for $10^3$ iterations under SGLD with varying mini-batch sizes.
The resulting trajectories are standardized to remove offsets and rescale the amplitudes.
Although the raw sequences are temporally dense, their fine-scale resolution contains redundant information.
To extract data suitable for dynamical characterization, we uniformly downsample each trajectory by a factor of~50, which preserves the coarse-grained evolution while suppressing short-term correlations.
The reduced trajectories are then partitioned into consecutive state pairs, where each pair represents the solver state and its one-step successor.
Repeating this procedure with $5\times10^5$ independent random initializations yields a large collection of state-transition samples.

\paragraph{Independent Components Analysis}
In the following, we provide an overview of a widely used independent component analysis (ICA) algorithm based on stochastic natural-gradient optimization~\cite{amari1995new}, and present its SGLD version~\cite{welling2011bayesian}.
ICA aims to recover statistically independent latent sources from their linear mixtures. It assumes that the observed signal $\mathbf{x}\in\mathbb{R}^D$ is generated as $\mathbf{x}=W^{-1}\mathbf{s}$, where the latent components $\mathbf{s}=(s_1,\dots,s_D)$ are mutually independent and typically exhibit non-Gaussian, heavy-tailed marginal distributions.
A probabilistic formulation of ICA thus introduces the likelihood
\begin{equation*}
p(\mathbf{x}, W) = |\det(W)| \left[ \prod_d p_d(\mathbf{w}_d^{\mathsf{T}} \mathbf{x}) \right]
\prod_{d_1,d_2} \mathcal{N}(W_{d_1 d_2}; 0, \lambda),
\end{equation*}
where a Gaussian prior is imposed on the weight matrix $W$ to promote regularity.

Following~\cite{amari1995new}, optimization of $W$ employs natural-gradient descent, 
which preconditions the Euclidean gradient by $W^{\mathsf T}W$ to account for the Riemannian geometry of the parameter manifold. 
In particular, if we choose $p_d(y_d)=1/(4\cosh^2(y_d/2))$ with $y_d=\mathbf{w}_d^{\mathsf T}\mathbf{x}$, 
the gradient takes the form:
\begin{equation*}
\mathcal{D}_W 
= \nabla_W \log p(X, W)\, W^{\mathsf T} W
= \left(
n I - \sum_{i=1}^{n}
\tanh\!\left(\tfrac{1}{2}\mathbf{y}_i\right)
\mathbf{y}_i^{\mathsf T}
\right) W
- \lambda W W^{\mathsf T} W,
\end{equation*}
where $\tanh(\cdot)$ acts elementwise. 

The discrete Langevin dynamics is then written as
\begin{equation*}
W^{(t+1)}  
= W^{(t)}  
+ \frac{\eta }{2}\mathcal{D}_W 
+ \sqrt{\eta}\, \boldsymbol{\xi}^{(t)} \sqrt{W_t^{\mathsf T} W_t},
\end{equation*}
where $\boldsymbol{\xi}^{(t)}\in \mathbb{R}^{D\times D}$ has i.i.d.\ entries drawn from $\mathcal{N}(0,1)$ 
and $\eta$ denotes the step size. 
The multiplicative factor $\sqrt{W_t^{\mathsf T}W_t}$ reflects the natural-gradient metric preconditioning. 
In the SGLD~\cite{welling2011bayesian} formulation, 
the full-data gradient is replaced by a mini-batch approximation of size $b$, 
yielding the update rule:
\begin{equation*}
W^{(t+1)}  
= W^{(t)}  
+ \frac{\eta n}{2b}
\sum_{i=1}^{b}
\left(
I - 
\tanh\!\left(\tfrac{1}{2}\mathbf{y}_{\gamma_i}\right)
\mathbf{y}_{\gamma_i}^{\mathsf T}
\right) W^{(t)}
- \frac{\eta \lambda}{2} W^{(t)} W^{(t)\mathsf T} W^{(t)}
+ \sqrt{\eta}\, \boldsymbol{\xi}^{(t)} \sqrt{W_t^{\mathsf T} W_t},
\end{equation*}
where $\{\gamma_1, \gamma_2, \ldots, \gamma_b\}$ is a randomly selected set of sample indices. 
This formulation provides a representative example of a nonconvex Bayesian inference problem with intrinsic manifold geometry,
making it an ideal benchmark for assessing the irreversibility and convergence behavior of nonlinear
SGLD under structured gradient preconditioning.

In this case, we construct a two-channel dataset consisting of 100 independent realizations, where one channel follows a heavy-tailed (high-kurtosis) distribution and the other follows a standard Gaussian distribution, resulting in a four-dimensional SGLD.
Each simulation starts from an initial condition uniformly sampled from $[-2, 2]^{2\times2}$ and is evolved for $10^3$ steps using SGLD with different batch sizes.
To mitigate oversampling and emphasize large-scale dynamics, each trajectory is downsampled by a factor of~10.
The resulting coarse-grained sequences are organized into successive state pairs representing consecutive SGLD updates.
Aggregating over $2\times10^5$ independent random seeds produces a comprehensive dataset suitable for learning dynamical transition patterns.

\paragraph{Additional Numerical Results}
\begin{figure}[!ht]
    \centering
    \includegraphics[width=\linewidth]{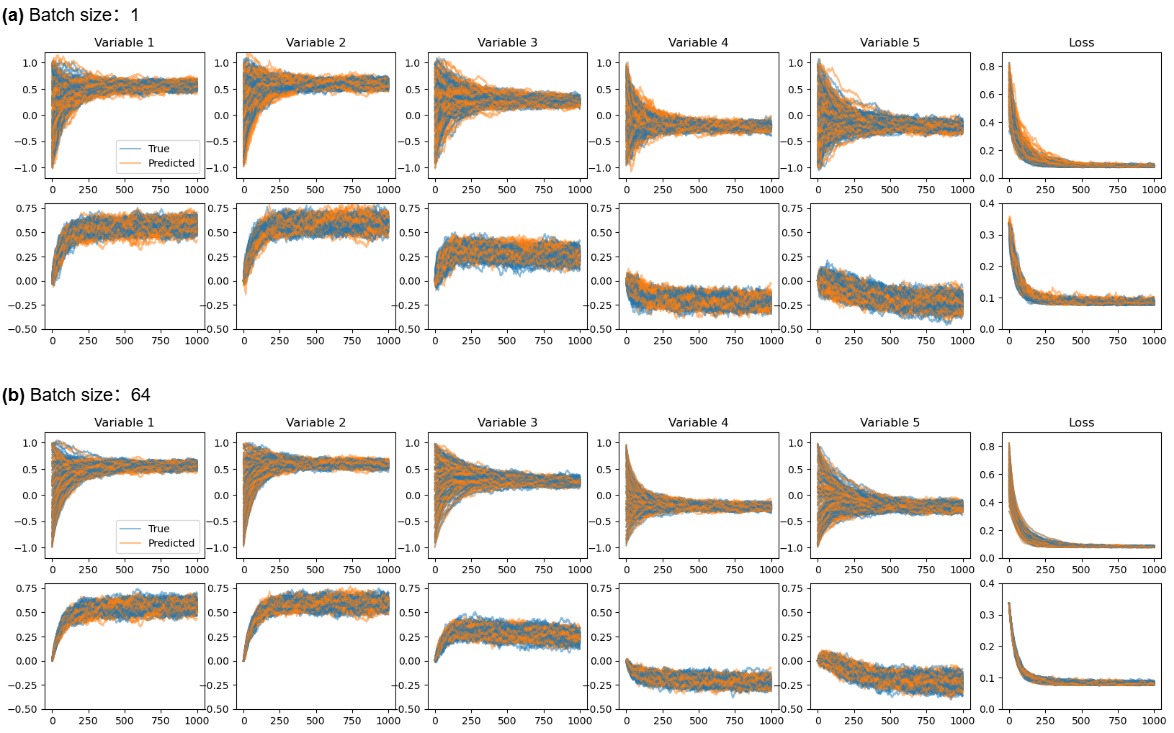}
    \caption{Predicted versus true least squares regression with SGLD trajectories under varying batch sizes and initialization.
    All trajectories are linearly rescaled to $[-1,1]$ and averaged over 500 independent runs.
 Panels (a) show results for batch size $1$, with either random initialization from $[-1,1]$ (top) or all coordinates set to zero (bottom). Panels (b) correspond to batch size $64$, presented in the same order.
}
    \label{fig:sgld_traj}
\end{figure}

\begin{figure}[!ht]
    \centering
    \includegraphics[width=\linewidth]{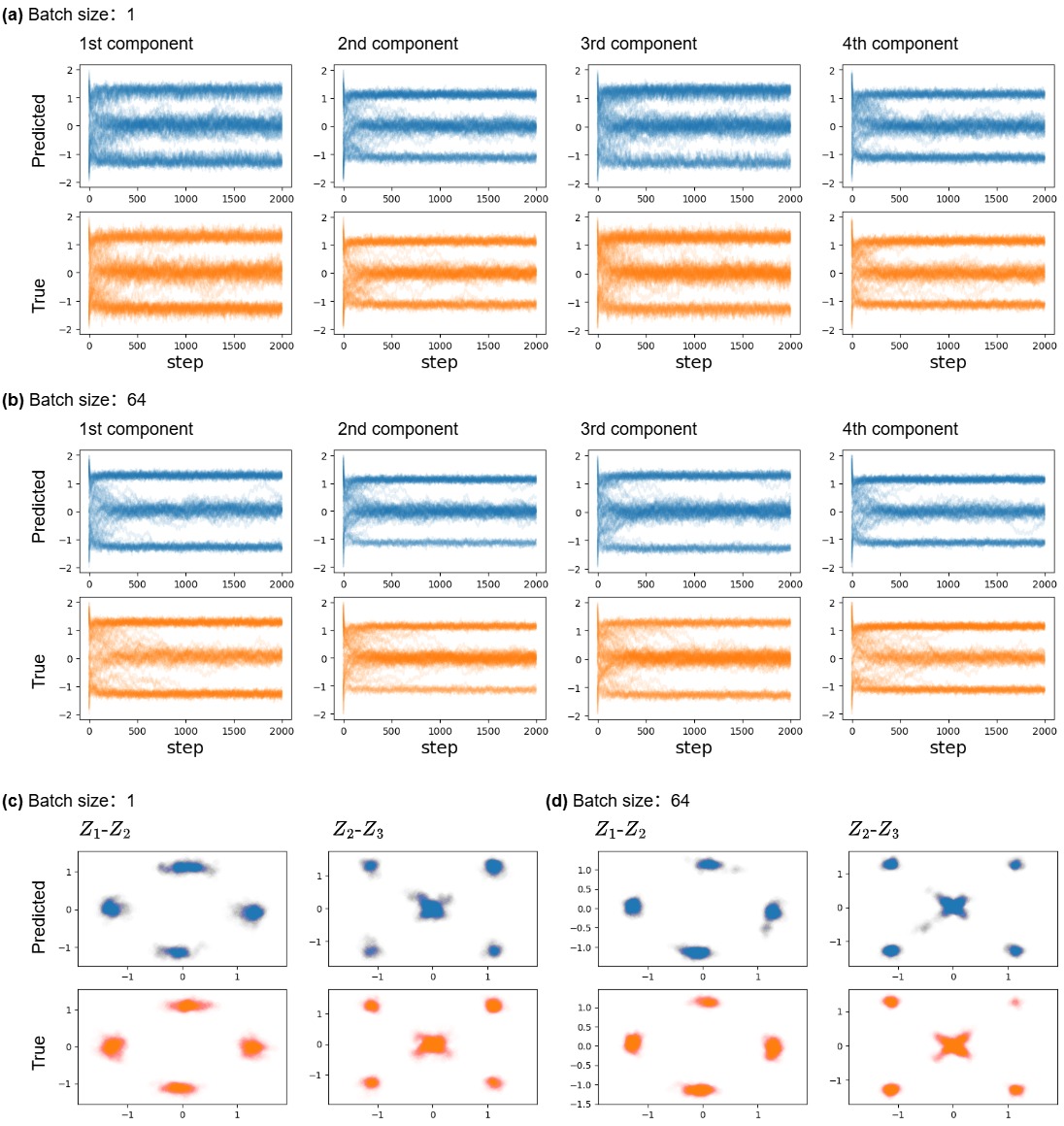}
    \caption{Predicted versus true ICA with SGLD trajectories and invariant distributions under varying batch sizes.
Time series plots (a, b) and scatter plots (c, d) demonstrate the consistency between the predicted and true latent trajectories across four ICA components. Panels (a) and (b) depict the temporal evolution over 2000 steps with batch sizes (1) and (64), respectively. Panels (c) and (d) compare the corresponding invariant distributions, showing scatter plots of \((Z_1, Z_2)\) and \((Z_2, Z_3)\).
}
    \label{fig:sgld_ica}
\end{figure}

Here the model is trained for 500 epochs using the Adam optimizer with a batch size of 100,000. Each batch consists of independently sampled paired data points.

To evaluate the robustness and reconstruction fidelity of {\ourmethodname} in the context of least-squares regression with SGLD, we examine its ability to reproduce both transient trajectories and output losses across diverse parameter settings (Fig.~\ref{fig:sgld_traj}).
{\ourmethodname} consistently recovers the temporal evolution of the stochastic dynamics under different batch sizes and initialization schemes, demonstrating that the learned representation captures the underlying generative mechanism rather than overfitting to specific training conditions.
This robustness enables a mechanistic interpretation of SGLD behavior directly from the learned model.

We further validate {\ourmethodname} on the nonlinear ICA benchmark (Fig.~\ref{fig:sgld_ica}).
Across varying batch sizes and initialization schemes, the model yields highly consistent reconstructions of both the transient trajectories and the invariant distributions.
These results confirm that {\ourmethodname} accurately captures the stochastic dynamics of the ICA with SGLD.

\paragraph{Theoretical Justification for EPR as a Measure of Sampling Error}
For convenience, we recall that the discrete-time dynamics of SGLD converge weakly, in the vanishing step-size limit, to a continuous-time diffusion process governed by the Itô SDE~\cite{li2019stochastic}:
\[
\dot{\Z} = -\nabla L(\Z) +  (\eta\,\Sigma(\Z) +  2 )^{\frac{1}{2}} \,\dot{\B}_t, \text{ with } \Sigma(\Z) = \text{Cov}[\frac{n}{b} \sum_{i=1}^b \nabla l_{\gamma_i}(\Z)],
\]
where $\{\gamma_i\}_{i=1}^{b}$ indexes a mini-batch of size $b$. By Theorem~\ref{the:appro}, there exists a matrix field \( W_b(\Z) \) of the form~\eqref{eq:app_W} and a potential function \( V_b(\Z) \) such that the drift term can be decomposed as:
\begin{equation*}
    -\nabla L(\Z) = -\left[ M_b(\Z) + W_b(\Z) \right] \nabla V_b(\Z) + \nabla \cdot M_b(\Z) + \nabla \cdot W_b(\Z),
\end{equation*}
where $M_b(\Z) = \frac{\eta}{2} \Sigma_b(\Z) + 1$. 
Denoting by $p_b$ the stationary density of SGLD with batch size $b$, and by $p_{\infty}$ that of the full-batch SGLD, we obtain:
\begin{align}
\label{eq:epr_final}
\dot{S}_{\text{tot}} = \int_{\mathbb{R}^D} f_{\text{irr}}^{\top} M^{-1}  f_{\text{irr}} \rho(\Z) d\Z = \mathbb{E}_{\Z \sim p_b} \left[  f_{\text{irr}}(\Z)^{\top} \left(\frac{\eta}{2} \Sigma_b(\Z) + 1\right)^{-1} f_{\text{irr}}(\Z)\right].
\end{align}
and
\begin{equation*}
    \nabla \ln p_b = -\nabla V_b, \quad \nabla \ln p_{\infty} = -\nabla L.
\end{equation*}

Next we introduce the Relative Fisher Information (RFI)  between distributions with densities $p$ and $q$.
This asymmetric divergence measures the discrepancy between their score functions:
\[
I(p, q) = \mathbb{E}_{\Z \sim p} \left[ \left\| \nabla \ln p(\Z) - \nabla \ln q(\Z) \right\|^2 \right].
\]
RFI provides an upper bound for the Wasserstein distance and the KL divergence according to HWI inequality~\cite{otto2000generalization} under the Bakry-\'{E}mery condition, with applications in information geometry and statistical physics~\cite{villani2021topics, mukherjee2018relative}.
For SGLD, the RFI is given by
\begin{align*}
I(p_b, p_{\infty}) &= \mathbb{E}_{\Z \sim p_b} \left[ \left\| \nabla \ln p_b(\Z) - \nabla \ln p_{\infty}(\Z) \right\|^2 \right] \\
&= \mathbb{E}_{\Z \sim p_b} \left[ \left\| -\nabla V_b(\Z) + \nabla L(\Z) \right\|^2 \right].
\end{align*}
Substituting~\eqref{eq: app_OnsagerHHD}, we obtain:
\begin{align*}
\nabla L(\Z) - \nabla V_b(\Z) &= \left[ M_b(\Z) + W_b(\Z) \right] \nabla V_b(\Z) - \left( \nabla \cdot M_b(\Z) + \nabla \cdot W_b(\Z) \right) - \nabla V_b(\Z) \\
&= \left[ \frac{\eta}{2} \Sigma_b(\Z) \nabla V_b(\Z) - \nabla \cdot \left( \frac{\eta}{2} \Sigma_b(\Z) \right) \right] +  f_{\text{irr}}(\Z) \\
&= \frac{\eta}{2} \left[ \Sigma_b(\Z) \nabla V_b(\Z) + \nabla \cdot \Sigma_b(\Z) \right] + f_{\text{irr}}(\Z).
\end{align*}
Thus,
\begin{align}
\label{eq:rfi_final}
I(p_b, p_{\infty}) &= \mathbb{E}_{\Z \sim p_b} \left[ \left\| \frac{\eta}{2} \left[ \Sigma_b(\Z) \nabla V_b(\Z) + \nabla \cdot \Sigma_b(\Z) \right] +  f_{\text{irr}}(\Z) \right\|^2 \right].
\end{align}

The difference between~\eqref{eq:epr_final} and~\eqref{eq:rfi_final} is attributed to the $\Sigma_b$ term,
reflecting the effect of minibatch-induced stochasticity and
diminishing as the batch size increases.
This observation provides a theoretical justification for the empirical finding that
both the sampling error and the global EPR consistently decrease with larger batch sizes.
In addition, since all quantities involved ($V_b$, $W_b$, $M_b$,
and consequently $f_{\mathrm{irr}}$ and $\Sigma_b$) are explicitly parameterized
and estimable within our framework {\ourmethodname},
\eqref{eq:rfi_final} can be computed directly from data and
thus serves as a principled scalar indicator of sampling error in SGLD.

\end{document}